\documentclass{statsoc}
\usepackage{graphicx}
\usepackage{natbib}
\usepackage{delarray}
\usepackage{hyperref}
\usepackage{rotating}
\usepackage{url}
\usepackage{colortbl,color}
\usepackage{anysize}
\usepackage{multirow,amssymb,amsmath, amsfonts}

\newtheorem{thm}{Theorem}
\newtheorem{lem}{Lemma}

\newtheorem{remark}{Remark}
\newtheorem{prop}{Proposition}
\def\boxit#1{\vbox{\hrule\hbox{\vrule\kern6pt\vbox{\kern6pt#1\kern6pt}\kern6pt\vrule}\hrule}}

\def\argmin{\mathop{\rm argmin}}
\def\tr{\mathop{\rm tr}}

\newcommand{\bB}{{\bf B}}
\newcommand{\bx}{{\bf x}}

\newcommand{\bX}{{\bf X}}

\newcommand{\bI}{{\bf I}}

\newcommand{\bw}{{\bf w}}

\newcommand{\bzero}{{\bf 0}}

\newcommand{\bXi}{\boldsymbol{\Xi}}

\newcommand{\bone}{{\bf 1}}
\newcommand{\bbeta}{\boldsymbol{\beta}}
\newcommand{\bmu}{\boldsymbol{\mu}}

\newcommand{\bxi}{\boldsymbol{\xi}}

\newcommand{\btheta}{\boldsymbol{\theta}}
\newcommand{\bSigma}{\boldsymbol{\Sigma}}

\newcommand{\bPsi}{\boldsymbol{\Psi}}

\newcommand{\bgamma}{\boldsymbol{\gamma}}

\newcommand{\bOmega}{\boldsymbol{\Omega}}

\newcommand{\bv}{\boldsymbol{v}}
\newcommand{\by}{\boldsymbol{y}}
\newcommand{\real}{I\kern-0.37emR}
\definecolor{annotxt}{rgb}{0,0,1}
\def\argmin{\mathop{\rm argmin}}
\newcommand{\bc}{{\mathcal{C}}}
\makeatletter
\newcommand{\rmnum}[1]{\romannumeral #1}
\newcommand{\Rmnum}[1]{\expandafter\@slowromancap\romannumeral #1@}
\makeatother

\renewcommand{\theenumi}{\roman{enumi}}

\title[ROAD to Classification]{A ROAD to Classification in High Dimensional Space}
\author[J. Fan, Y. Feng and X. Tong]{Jianqing Fan$^{[1]}$, Yang Feng$^{[2]}$ and Xin Tong$^{[1]}$}
\address{$^{[1]}$Department of Operations Research $\&$ Financial
Engineering, Princeton University, Princeton,
New Jersey 08544, U.S.A.}
\address{$^{[2]}$Department of
Statistics, Columbia University,
New York, NY 10027, U.S.A.
}
\email{jqfan@princeton.edu}
\email{yangfeng@stat.columbia.edu}
\email{xtong@princeton.edu}

\begin{document}

\begin{abstract}
For high-dimensional classification, it is well known that
naively performing the  Fisher discriminant rule leads to poor results
due to diverging spectra and noise accumulation.
Therefore,  researchers proposed independence rules to circumvent
the diverging spectra, and sparse independence rules to mitigate the issue of noise
accumulation. However, in biological applications, there are often a group of correlated genes
responsible for clinical outcomes, and the use of the covariance information
can significantly reduce misclassification rates.  In theory the extent of such error rate reductions is unveiled by comparing the misclassification rates of
the Fisher discriminant rule and the independence rule.
To materialize the gain based on finite samples,
a Regularized Optimal Affine Discriminant (ROAD) is proposed.  ROAD
selects an increasing number of  features as the regularization relaxes.
Further benefits can be achieved when a screening method
is employed to narrow the feature pool before hitting the ROAD.
An efficient Constrained Coordinate Descent algorithm (CCD)
is also developed to solve the associated optimization problems.  Sampling properties of oracle type are established.
Simulation studies and real data analysis
support our theoretical results and demonstrate the advantages
of the new classification procedure under a variety of correlation structures. A delicate result on continuous piecewise linear solution path for the ROAD optimization problem at the population level justifies the linear interpolation of the CCD algorithm.
\end{abstract}

\section{Introduction}
Technological innovations have had deep impact on society and on various areas of scientific research. High-throughput data from microarray and proteomics technologies are frequently used in many contemporary statistical studies. In the case of microarray data, the dimensionality is frequently in thousands or beyond, while the sample size is typically in the order of tens. The large-$p$-small-$n$ scenario poses challenges for the classification problems.  We refer to  \citet{FanLv2010} for an overview of statistical challenges associated with high dimensionality.

When the feature space dimension $p$ is very high compared to the sample size $n$, the Fisher discriminant rule performs poorly due to diverging spectra as demonstrated by \citet{BickelLevina-04}. These authors showed that the independence rule in which the covariance structure is ignored performs better than the naive Fisher rule (NFR) in the high dimensional setting.  \citet{FanFan-08} demonstrated further that even for the independence rules, a procedure using all the features can be as poor as
random guessing due to noise accumulation in estimating population centroids
in high-dimensional feature space. As a result, \citet{FanFan-08} proposed the Features Annealed Independence Rule (FAIR) that selects a subset of important features for classification. \citet{Dudoit-2002} reported that for microarray data, ignoring correlations between genes leads to better classification results. \citet{Tibshirani-02} proposed the Nearest Shrunken Centroid (NSC) which likewise employs the working independence structure. Similar problems are also studied in the machine learning community such as \citet{Domingos-1997} and \citet{Lewis98}.

In microarray studies, correlation among different genes is an essential characteristic of the data and usually not negligible. Other examples include proteomics, and metabolomics data where correlation among biomarkers is commonplace. More details can be found in \citet{Ackermann-2009}. Intuitively, the independence assumption among genes leads to loss of critical information and hence is suboptimal. We believe that in many cases, the crucial point is not whether to consider correlations, but how we can incorporate the covariance structure into the analysis with a bullet proof vest against diverging spectra and significant noise accumulation effect.

The setup of the objective classification problem is now introduced. We assume in the following that the variability of data under consideration can be  described reasonably well by the means and variances.  To be more precise, suppose that random variables representing two classes $\bc_1$ and $\bc_2$ follow $p$-variate normal distributions: $\bX|Y=1\sim\mathcal{N}_p(\bmu_1,\bSigma)$ and
$\bX|Y=2\sim\mathcal{N}_p(\bmu_2,\bSigma)$
respectively. Moreover, assume $\mathbb{P}(Y=1)=1/2$. This Gaussian discriminant analysis setup is known for its good performance despite its rigid model structure.
%Denote the parameter by $\btheta=(\bmu_1,\bmu_2,\bSigma)$.
For any linear discriminant rule
\begin{equation} \label{eq1}
    \delta_{\bw}(\bX)=\mathbb{I}\{\bw^T(\bX-\bmu_a)>0\},
\end{equation}
where $\bmu_a=(\bmu_2+\bmu_1)/2$, and $\mathbb{I}$ denotes the indicator function with  value $1$ corresponds to assigning $\bX$ to class $\bc_2$ and $0$ class $\bc_1$, the misclassification rate of the (pseudo) classifier $\delta_{\bw}$ is
\begin{equation}  \label{eq2}
  W(\delta_{\bw})=\frac{1}{2}P_2(\delta_{\bw}(\bX)=0)+\frac{1}{2}P_1(\delta_{\bw}(\bX)=1)=1-\Phi(\bw^T\bmu_d/(\bw^T\bSigma\bw)^{1/2}),
\end{equation}
where $\bmu_d=(\bmu_2-\bmu_1)/2$, and $P_i$ is the conditional distribution of $\bX$ given its class label $i$. We will focus on such linear classifier $\delta_{\bw}(\cdot)$, and the mission is to find a good data projection direction $\bw$.
%The misclassification rate for observation from class $\bc_1$ can be easily obtained by interchanging the roles of two classes; hence we assume that the new observation is always from class $\bc_2$, and focus on the properties of $W(\delta_{\bw})$.
Note that the Fisher discriminant
\begin{equation} \label{eq3}
   \delta_F (\bX) = \mathbb{I}\{(\bSigma^{-1}\bmu_d)^T(\bX-\bmu_a)>0\}
\end{equation}
is the \emph{Bayes rule}. %There is an equivalent derivation of the Fisher discriminant, which does not involve Gaussian assumptions. We would skip it for now, and come back to this point when we extend our method to multi-class learning scenarios.
There are two fundamental difficulties in applying the Fisher discriminant whose missclassification rate is
\begin{equation} \label{eq4}
1 - \Phi \left ( (\bmu_d^T\bSigma^{-1}\bmu_d )^{1/2} \right ).
\end{equation}
The first difficulty arises from the noise accumulation effect in estimating the population centroids \citep{FanFan-08} when $p$ is large.  The second challenge is more severe:  estimating the inverse of covariance matrix $\bSigma$ when $p > n$ \citep{BickelLevina-04}.  As a result, much previous researches focus on the independence rules, which act as if $\bSigma$ is diagonal.  However, correlation matters!

To illustrate this point, consider a case when $p=2$. These two features can be selected from the original thousands of features, and we can estimate the correlation between two variables with reasonable accuracy.  Let
\[ \bSigma= \left( \begin{array}{ccc}
1 & \rho \\
\rho & 1 \end{array} \right),\]
where $\rho\in[0,1)$ and $\bmu_d=(\mu_1, \mu_2)^{T}$. Without loss of generality, assume $|\mu_1| \geq |\mu_2| >0$.  The misclassification rate of Fisher discriminant depends on
\begin{equation} \label{eq4a}
   \Delta_p (\rho) =\bmu_d^T\bSigma^{-1}\bmu_d = \frac{1}{1-\rho^2}(\mu_1^2+\mu_2^2-2\rho\mu_1\mu_2).
\end{equation}
Note that $$\Delta_p'(\rho)>0 \Leftrightarrow \mu_1\mu_2\rho^2-(\mu_1^2+\mu_2^2)\rho+\mu_1\mu_2<0.$$
Therefore, when $\mu_1\mu_2<0$, $\Delta'_p(\rho)>0$ for all $\rho\in[0,1)$.
On the other hand, when $\mu_1\mu_2>0$, $\Delta_p(\rho)$ decreases on $\rho\in (0, \frac{\mu_2}{\mu_1})$, and increases on $(\frac{\mu_2}{\mu_1},1)$.  Notice that when $\rho\rightarrow 1$, $\Delta_p\rightarrow \infty$ regardless of signs for $\mu_1 \mu_2$,  which in turn leads to vanishing classification error.
On the other hand, if we use independence rule (also called naive Bayes rule), the optimal misclassification rate
\begin{equation}  \label{eq5}
1 - \Phi \left ( \frac{\|\bmu_d\|_2^2}{(\bmu_d^T\bSigma\bmu_d)^{1/2}} \right )
\end{equation}
depends on $\Gamma(\rho) = \|\bmu_d\|_2^4 / \bmu_d^T\bSigma\bmu_d$, which is monotonically decreasing for $\rho\in[0,1)$, with the limit $(\mu_1^2+\mu_2^2)^2/(\mu_1+\mu_2)^4$ that is smaller than unity when $\mu_1$ and $\mu_2$ have the same sign.
Hence, the optimal classification error using the independence rule actually increases as correlation among features increases.

The above simple example shows that by incorporating correlation information, the gain in terms of classification error can be substantial.  Elaboration on this point in more realistic scenarios is provided in Section 2.  Now it seems wise to use at least a part of covariance structure to improve the performance of a classifier.
%While the whole covariance matrix $\bSigma$ can not be estimated well for large $p$, the low order submatrices of the covariance matrix can be estimated well (e.g. the correlation of any two features is easy to estimate). Intuitively, it is too extreme to jump from the naive ambition to utilize all covariance information in the Fisher discriminant to complete ignorance of  any covariance structure in the independence rules. There should be some good compromise in the middle; indeed, a well estimated low-order submatrix can be employed to improve the performance of the classifiers.\textcolor{red}{the next replaces what come after "While the whole covariance...."}
So there is a need to estimate the covariance matrix $\bSigma$. Without structural assumptions on $\bSigma$, the pooled sample covariance $\hat{\bSigma}$ is one natural estimate. But for $p > n$, it is not considered as a good estimate of $\bSigma$ in general.  We are lucky here because our mission is not constructing a good estimate of the covariance matrix, but finding a good direction $\bw$ that leads to a good classifier.  To mimic the optimal data projection direction $\bSigma^{-1}\bmu_d$,  we do not adopt a direct plug-in approach, simply because it is unlikely that a product is a good estimate when at least one of its components is not.  Instead, we find the data projection direction $\bw$ by directly minimizing the classification error subject to a capacity constraint on $\bw$. From a broad spectrum of simulated and real data analysis, we are convinced that this approach leads to a robust and efficient sparse linear classifier.

Admittedly, our work is far from the first to use covariance for classification; support vector machines \citep{Vapnik-1995}, for example, implicitly utilize covariance between covariates. Another notable work is ``shrunken
centroids regularized discriminant analysis'' (SCRDA) \citep{Guo-2005}, which calls for a version of regularized sample covariance matrix $\hat{\bSigma}_{\text{reg}}$, and soft-thresholds on $\hat{\bSigma}^{-1}_{\text{reg}}\hat{\bx}_i$. \cite{ShaoWang11} consider a sparse linear discriminant analysis, assuming the sparsity on both the covariance matrix and the mean difference vector so that they can be regularized.  They show that such a regularized estimator is asymptotically optimal under some conditions.
%The penalization idea also exists in   \citet{Friedman1989},
%\citet{Hastie1995}
%and more recently \citet{Witten2010}.
However, to the best of our knowledge, this work is the first to select features by directly optimizing the misclassification rates, to explicitly use un-regularized sample covariance information, and to establish the oracle inequality and risk approximation theory.

%******Can talk something about feature selection.............

There is a huge literature on high dimensional classification. %One popular method of dimensionality reduction is projection.
Examples include principal component analysis in \citet{BairHast-2006} and \citet{ZouHastieTibs-2006}, partial least squares in \citet{Nguyen-2002}, \citet{HuangPan-2003} and \citet{Boulesteix-2004}, and sliced inverse regression in \citet{Li-1991-sliced-inverse-regression} and \citet{Antoniadis-2003}.

The rest of our paper is organized as follows.  Section \ref{sec::special-case} provides some insights on the performances of naive Bayes, Fisher discriminant and restricted Fisher discriminants. In Section \ref{sec::ROAD}, we propose the Regularized Optimal Affine Discriminant (ROAD) and variants of ROAD. An efficient algorithm Constrained Coordinate Descent (CCD) is constructed in Section \ref{sec::ccd}. Main risk approximation results and continuous piecewise linear property of the solution path are established in Section \ref{sec::asymptoticproperty}. We conduct simulation and empirical studies in Section \ref{sec::numerical-study}.
% Section \ref{sec:multi class} provides an outline to extend ROAD to multi-class setting.
 A  discussion is given in Section \ref{sec::summary}, and all proofs are relegated to the appendix.

\section{Naive Bayes and Fisher Discriminant}\label{sec::special-case}

To compare the naive Bayes and Fisher discriminant at the population level, we assume without loss of generality that variables have been marginally standardized so that $\bSigma$ is a correlation matrix.  Recall that the naive Bayes discriminant has error rate (\ref{eq5}) and the Fisher discriminant has error rate (\ref{eq4}).  Let $\Gamma_p = \|\bmu_d\|_2^4 /\bmu_d^T\bSigma\bmu_d$ and
$\Delta_p = \bmu_d^T\bSigma^{-1} \bmu_d$. Denote by $\{\lambda_i\}_{i=1}^p$ the eigenvalues and $\{\bxi_i\}_{i=1}^p$ eigenvectors of the matrix $\bSigma$.  Decompose
\begin{equation} \label{eq6}
    \bmu_d=a_1\bxi_1+\cdots + a_p\bxi_p,
\end{equation}
where $\{a_i\}_{i=1}^p$ are the coefficients of $\bmu_d$ in this new orthonormal basis $\{\bxi_i\}_{i=1}^p$.  Using the decomposition (\ref{eq6}), we have
\begin{equation}  \label{eq7}
\Delta_p = \sum_{j=1}^p  a_j^2/\lambda_j, \qquad \Gamma_p = \Bigl (\sum_{j=1}^p a_j^2 \Bigr )^2  /\sum_{j=1}^p \lambda_j a_j^2 .
\end{equation}
The relative efficiency of Fisher discriminant over naive Bayes is characterized by $\Delta_p/\Gamma_p$.  By the Cauchy-Schwartz inequality,
$$
\Delta_p / \Gamma_p \geq 1.
$$
The naive Bayes method performs as well as the Fisher discriminant only when $\bmu_d$ is an eigenvector of $\bSigma$.

In general, $\Delta_p/\Gamma_p$ can be much larger than unity.  Since $\bSigma$ is the correlation matrix,  $\sum_{j=1}^p \lambda_j = \tr(\bSigma) = p$. If $\bmu_d$ is equally loaded on $\bxi_j$, then the ratio
\begin{equation}  \label{eq8}
    \Delta_p / \Gamma_p = p^{-2} \sum_{j=1}^p \lambda_j  \sum_{j=1}^p \lambda_j^{-1}
    = p^{-1}  \sum_{j=1}^p \lambda_j^{-1}.
\end{equation}
More generally, if $\{a_j\}_{j=1}^p$ are realizations from a distribution with the second moment $\sigma^2$, then by the law of large numbers,
$$
 \sum_{j=1}^p  a_j^2\lambda_j^{-1} \approx \sigma^2  \sum_{j=1}^p 1/\lambda_j, \quad
p^{-1} \sum_{j=1}^p  a_j^2 \approx \sigma^2, \quad  \sum_{j=1}^p  \lambda_j a_j^2 \approx \sigma^2  \sum_{j=1}^p \lambda_j.
$$
Hence, (\ref{eq8}) holds approximately in this case.  In other words, the right hand side of (\ref{eq8}) is approximately the relative efficiency of the Fisher discriminant over the naive Bayes. Now suppose further that half of the eigenvalues of $\bSigma$ are $c$ and the other half are $2-c$. Then, the right hand side of (\ref{eq8}) is
$(c^{-1} + (2-c)^{-1})/2$.  For example when the condition number is $10$, this ratio is about $3$. A high ratio translates into a large difference in error rates:  $1 - \Phi(\Gamma_p^{1/2})$ for independence rule is much larger than $1 - \Phi(3\Gamma_p^{1/2})$ for Fisher discriminant.  For example, when $\Gamma_p^{1/2} = 0.5$, we have 30.9\% and 6.7\% error rates respectively for the naive Bayes and Fisher discriminant.

\begin{figure}

\begin{center}
\includegraphics[scale=0.6]{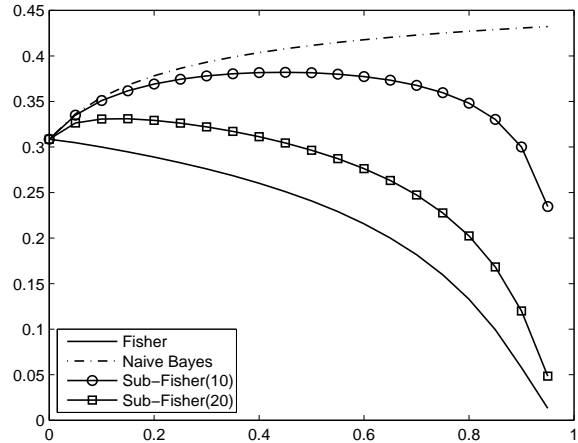}
\end{center}
\caption{Misclassification rates of Fisher discriminant, naive Bayes and restricted Fisher rules  (10 and  20 features, respectively) against $\rho$.}\label{fig1}
\end{figure}

To put the above arguments under a visual inspection, consider a case in which $p=1000$, $\bmu_d=(\bmu_s^T, \mbox{\bf 0}^T)^{T}$ with $\bmu_s=(1, 1, 1, 1, 1, 2, 2, 2, 2,2)^T$ and $\bSigma$ equals the equi-correlation matrix with pairwise correlation $\rho$.  The vector $\bmu_d$ simulates the case in which  10 genes out of 1000 express mean differences.
Figure~\ref{fig1} depicts the theoretical error rates of the Fisher discriminant  and the naive Bayes rule  as functions of $\rho$.

It is not surprising that the Fisher discriminant rule performs significantly better than the naive Bayes as $\rho$ deviates away from zero.  The error rate of the naive Bayes actually increases with $\rho$, whereas the error rate of the Fisher discriminant tends to zero as $\rho$ approaches 1. This phenomenon is the same as what was shown analytically through the toy example in Section 1.  To mimic Fisher discriminant by a plug-in estimator, we  need to estimate $\bSigma^{-1}\bmu_d$ with reasonable accuracy. This mission is difficult if not impossible.  On the other hand, imitating a weaker oracle is more manageable.  For example, when the samples are of reasonable size, we can select the 10 variables with differences in means by applying a two-sample $t$-test.  Restricting to the best linear classifiers based on these $s=10$ variables, we have the optimal error rate
$$
  1-\Phi((\bmu_{s}^{T}\bSigma_{s}^{-1}\bmu_s)^{1/2}),
$$
where the classification rule is $\delta_{\bw^R}$ and $\bw^R=((\bSigma_s^{-1}\bmu_s)^T, \mbox{\bf 0}^T)^T$. The performance of this oracle classifier is depicted by the sub-Fisher (10 features) in  Figure~\ref{fig1}.   It performs much better than the naive Bayes method.
One can also employ the naive Bayes rule to the restricted feature space, but this method has exactly the same performance as the naive Bayes method in the whole space.  Thus, the restricted Fisher discriminant outperforms both the naive Bayes method with restricted features and the naive Bayes method using all features.

Mimicking the performance of the restricted Fisher discriminant is feasible.  Instead of estimating a $1000 \times 1000$ covariance matrix, we only need to gauge a $10 \times 10$ submatrix.  However, this restricted Fisher rule is not powerful enough, as shown in Figure 1.  We can improve its performance by including 10 most correlated variables to each of those selected features to further account
for the correlation effect, giving rise to a 20-dimensional feature space.  Since the variables are equally correlated in this example, we are free to choose any 10 variables among the other 990.
The performance of such an enlarged restricted Fisher discriminant is represented by sub-Fisher (20 features) in Figure~\ref{fig1}.  It performs closely to the Fisher discriminant which uses the whole feature space, and it is feasible to implement with finite samples.

\section{Regularized Optimal Affine Discriminant}\label{sec::ROAD}

The misclassification rate of Fisher discriminant is $1-\Phi(\Delta_p^{1/2})$, where $\Delta_p=\bmu_d^T\bSigma^{-1}\bmu_d$. However, for high dimensional data, it is impossible to achieve such a performance empirically. Among other reasons, the estimated covariance matrix $\hat\bSigma$ is ill-conditioned or not invertible. One solution is to focus only on the $s(<<p)$ most important features for classification.  Ideally, the best $s$ features should be the ones with the largest $\Delta_s$ among all $p\choose s$ possibilities, where $\Delta_s$ is the counterpart of $\Delta_p$
when only $s$ variables are considered.   Naive search for the best subset of size $s$ is NP-hard. Thus, we develop a regularized method to circumvent these two problems.

\subsection{ROAD}

%Let $R(\bw)=\bw^T\bSigma\bw$ and its empirical version $R_n(\bw)=\bw^T\hat\bSigma\bw$, where $\hat\bSigma$ is an estimate of $\bSigma$.

Recall that by \eqref{eq2}, minimizing the classification error $W(\delta_{\bw})$ is the same as maximizing $\bw^T\bmu_d/(\bw^T\bSigma\bw)^{1/2}$, which is equivalent to minimizing $\bw^T\bSigma\bw$ subject to $\bw^T\bmu_d=1$.
We would like to add a penalty function for capacity control.  There are many ways to do regularization; for the literature on penalized methods, refer to LASSO \citep{Tibs:regr:1996},  SCAD \citep{FanLi2001},  Elastic net \citep{ZouHastie05},  MCP
\citep{Zhang09} and related methods \citep{Zou06,ZouLi08}. As our primary interest is classification error (the risk of the procedure),  an $\mbox{L}_1$ constraint $\|\bw\|_1\leq c$ is added for regularization, so the problem can be recast as
\begin{equation}\label{eq9}
\bw_{c}=\argmin_{\|\bw\|_1\leq c, \bw^T\bmu_d=1}\bw^T\bSigma\bw.
\end{equation}
We name the classifier $\delta_{\bw_c}(\cdot)$ the Regularized Optimal Affine Discriminant(ROAD).  The existence of a feasible solution in \eqref{eq9} dictates
\begin{equation} \label{eq10}
    c\geq 1/{\max_{1\leq i\leq p}|\mu_{d,i}|}.
\end{equation}

When $c$ is small, we obtain a sparse solution and achieve feature selection using covariance information. When $c\geq \|\bSigma^{-1}\bmu_d\|_1 / \bmu_d^T\bSigma^{-1}\bmu_d$, the $L_1$ constraint is no longer binding and $\delta_{\bw_c}$ reduces to the Fisher discriminant, which can be denoted by $\delta_{\bw_{\infty}}$ ($=\delta_{F}$). Therefore we have provided a family of linear discriminants, indexed by $c$, using from only one feature to all features. In some applications such as portfolio selection, the choice of $c$ reflects the investor's tolerance upper bound on gross exposure.  In other applications, when the user does not have a such a preference,  the choice of $c$ can be data-driven. To accommodate both application scenarios, we propose a coordinate descent algorithm (Section \ref{sec::ccd}) to implement our ROAD proposal.

\subsection{Variants of ROAD}

At the sample level, NSC \citep{Tibshirani-02} and FAIR \citep{FanFan-08} both use shrunken versions of standardized mean difference to find the $s$ features.  In the same spirit, we consider
the following Diagonal Regularized Optimal Affine Discriminant(D-ROAD) $\delta_{\bw_c^I}$, where
\begin{equation}\label{eq11}
\bw^I_{c}=\argmin_{\|\bw\|_1\leq c, \bw^T\bmu_d=1}\bw^T\mbox{diag}(\bSigma)\bw.
\end{equation}
The D-ROAD will be compared with NSC \citep{Tibshirani-02} and FAIR \citep{FanFan-08}  in the simulation studies, and all these independence based rules will be compared with ROAD and its two variants defined below.

%As to be detailed in Section 5, we actually refer to the sample versions of (\ref{eq10}) and (\ref{eq11}) as ROAD and RAID, respectively.  ROAD and RAID can also be applied to a subset of features.

A  screening-based variant  (to be proposed) of ROAD aims at mimicking the performance of sub-Fisher (10 features) in Figure 1. A fast way to select features is the independence screening, which uses the marginal information such as the two-sample $t$-test. We can also enlarge the selected feature subspace by incorporating the features which are most correlated to what have been chosen. This additional variant of ROAD tracks the performance of sub-Fisher (20 features) in Figure 1.  We will refer to the two variants of ROAD as S-ROAD1 and S-ROAD2. More description of these procedures, along with their theoretical properties and numerical investigations, will be detailed in Sections 5 and 6.

A hint of the rationale behind including correlated features that do not show a
difference in means between the two classes, is revealed through the two-feature example in the introduction.  Suppose $\mu_2 = 0$.  Then, by (\ref{eq4a}), the power of the discriminant using two features is $1-\Phi(\Delta_2^{1/2})$ where  $\Delta_2 = \mu_1^2 / (1-\rho^2)$, whereas with the first feature alone the misclassification rate is $1-\Phi(\Delta_1^{1/2})$ where $\Delta_1 = \mu_1^2$.  Therefore when the correlation $|\rho|$ is large, using two correlated features is far more powerful than employing only one feature, even though the second feature has no marginal discrimination power. More intuition is granted by this observation:  at the population level, the best s features are not necessarily those with largest standardized mean differences. In other words, with the two class Gaussian model in mind, when $\bSigma$ is the correlation matrix, the most powerful s features for classification are not necessarily the coordinates of $\bmu_d$ with largest absolute values. This is illustrated by the next stylized example.

Let $\bX|Y=0\sim\mathcal{N}(\bmu_1, \bSigma)$ and $\bX|Y=1\sim\mathcal{N}(\bmu_2, \bSigma)$, where $\bmu_1=(0,0,0)^T$,  $\bmu_2=(4,0.5,1)^T$,  and
\[ \bSigma = \left( \begin{array}{ccc}
1 & -0.25 & 0 \\
-0.25 & 1 & 0 \\
0 & 0 & 1 \end{array} \right).\]
Suppose the objective is to choose 2 out of 3 variables for classification.   If we rank features by marginal information, for example by the absolute value of standardized mean differences, then we would choose the 1st and 3rd features.  On the other hand, denote $\bmu_{d, ij}$ the mean difference vector for features $i$ and $j$, $\bSigma_{ij}$ the covariance matrix of features $i$ and $j$, then the classification power using features $i$ and $j$ depends on $\Gamma_{ij}=\bmu_{d,ij}^T\bSigma_{ij}^{-1}\bmu_{d, ij}$. Simple calculation leads to
$$
\Gamma_{12}=18.4>17=\Gamma_{13}\,.
$$
Hence the most powerful two features for classification are not the 1st and 3rd.

\section{Constrained Coordinate Descent}\label{sec::ccd}
With a Lagrangian argument, we reformulate problem (\ref{eq9}) as
\begin{eqnarray}\label{eq12}
{\bar\bw}_{\lambda} = \argmin_{\bw^{T}\bmu_d = 1}\frac{1}{2}\bw^{T}\bSigma \bw + \lambda\|\bw\|_1.
\end{eqnarray}
In this section, we  propose a Constrained Coordinate Descent (CCD) algorithm  that is tailored  for solving our minimization problem with linear constraints.  Optimization \eqref{eq12} is a constrained quadratic programming problem and can be solved by existing softwares such as MOSEK.  Although these softwares are well regarded in practice, they are slow for our application.   The structure
of \eqref{eq12} could be exploited in order to obtain a more efficient algorithm. In line with the
LARS algorithm, we will exploit the fact that the solution path has a piecewise-linear
property.

In the compressed sensing literature, it is common to replace an affine constraint by a quadratic penalty.  We borrow this idea and consider the following approximation to \eqref{eq12}:
\begin{eqnarray}\label{eq13}
\tilde{\bw}_{\lambda,\gamma} = \argmin \frac{1}{2}\bw^{T}\bSigma \bw + \lambda\|\bw\|_1 + \frac{1}{2}\gamma (\bw^{T}\bmu_d - 1)^2\,.
\end{eqnarray}
In practice, we replace $\bSigma$ by the pooled sample covariance $\hat{\bSigma}$ and $\bmu$ by the sample mean difference vector $\hat{\bmu}_d$.  By Theorem 6.7 in \cite{Ruszczynski06}, we have
$$
 \tilde{\bw}_{\lambda,\gamma}\rightarrow \bar{\bw}_{\lambda}\mbox{ when } \gamma \rightarrow \infty.
$$
Note that we do not have to enforce the affine constraint strictly,  because it only serves to normalize our problem. In the optimization problem (\ref{eq13}), when $\lambda = 0$, the solution $\tilde{\bw}_{0, \gamma}$ is always in the direction of $\bSigma^{-1} \bmu_d$, the Fisher discriminant, regardless of the value of $\gamma$. In addition, this observation is confirmed in the data analysis (Section 6.2) by the insensitivity of choice for $\gamma$.  Therefore we hold $\gamma$ as a constant in practice.

We solve \eqref{eq13} by coordinate descent.  Non-gradient algorithms seem to be less popular for convex optimization.  For instance, the popular textbook \emph{Convex Optimization} by \cite{Boyd04} does not even have a section on these methods.  %Among the two common approaches when closed form gradients are not available, we opt for coordinate descent over numerical differentiation, for reasons that will become clear soon.
Coordinate descent method is an algorithm, in which the $p$ search directions are just unit vectors $e_1, \cdots, e_p$, where $e_i$ denotes the $i$th element in the standard basis of $\mathbb{R}^p$.  These unit vectors are used as search directions in each search cycle until some convergence criterion is met.

What makes the coordinate descent algorithm particularly attractive for \eqref{eq13} is that there is an explicit formula for each coordinate update. For a given $\gamma$, fix  $\tau$ and $K$,  then do the optimization on a grid (of log-scale) of $\lambda$ values: $\tau\lambda_{\max}=\lambda_K<\lambda_{K-1}<\cdots<\lambda_1=\lambda_{\max}$. The $\lambda_{\max}$ is the minimum $\lambda$ value such that no variables enter the model; this is analogous to the minimum requirement on $c$ in (\ref{eq10}). In our implementation, we take $\tau=0.001$ and $K=100$.
The problem is solved backwards from $\lambda_{\max}$. When $\lambda=\lambda_{i+1}$, we use the solution from $\lambda=\lambda_{i}$ as the initial value.  This kind of ``warm start" is very effective in improving computational efficiency.

Consider a coordinate descent step to solve (\ref{eq13}). Without loss of generality, suppose that $\tilde{w}_j$ for all $j\geq 2$ are given, and we need to optimize with respect to $w_1$.  The objective function now becomes
\begin{align*}
g(w_1)&= \frac{1}{2} \begin{pmatrix}
  w_1^T&\tilde \bw_2^T
\end{pmatrix}\begin{pmatrix}
\Sigma_{11} & \bSigma_{12} \\
\bSigma_{21} & \bSigma_{22}  \\
\end{pmatrix}\begin{pmatrix}
  w_1\\ \tilde \bw_2
\end{pmatrix}+\lambda |w_1| + \lambda |\tilde \bw_2|_1+ \frac{1}{2}\gamma (\bw^{T}\bmu_d - 1)^2.
\end{align*}
When $w_1\neq 0$, we have
\begin{align*}
g'(w_1)&=  \Sigma_{11} w_1+ \bSigma_{12} \tilde \bw_2 + \lambda \mbox{ sign}(w_1) + \gamma (\bw^{T}\bmu_d - 1) \mu_{d1}  \\
&=  (\Sigma_{11}+\gamma \mu_{d1}^2) w_1+ (\bSigma_{12} + \gamma \mu_{d1} \bmu_{d2}^T) \tilde \bw_2 + \lambda \mbox{ sign}(w_1) -\gamma \mu_{d1}.
\end{align*}
By simple calculation \citep{Dono:John:idea:1994}, the coordinate-wise update has the form
$$
\tilde w_1= \frac{S\left(\gamma \mu_{d1} -(\bSigma_{12}+\gamma \mu_{d1}\bmu_{d2}^T)\tilde\bw_2, \lambda \right)}{\Sigma_{11}+\gamma \mu_{d1}^2},
$$
where $S(z, \lambda) = \mbox{sign}(z)(|z|-\lambda)^+ $ is the soft-thresholding operator.

Now, we consider the convergence property of the coordinate descent algorithm. Here, although the objective function is not strictly convex, it is strictly convex in each of the coordinates.

To show $g(w_1)$ is strictly convex in $w_1$, we decompose it as follows: $$g(w_1)=g_1(w_1)+g_2(w_1),$$ where $g_2(w_1)=\lambda |w_1|$ and
\begin{align*}
g_1(w_1)&= \frac{1}{2} \begin{pmatrix}
  w_1^T&\tilde \bw_2^T
\end{pmatrix}\begin{pmatrix}
\Sigma_{11} & \bSigma_{12} \\
\bSigma_{21} & \bSigma_{22}  \\
\end{pmatrix}\begin{pmatrix}
  w_1\\ \tilde \bw_2
\end{pmatrix} + \lambda |\tilde \bw_2|_1+ \frac{1}{2}\gamma (\bw^{T}\bmu_d - 1)^2\,.
\end{align*}
Note that $g_1(w_1)$ is a quadratic function of $w_1$ and $g_1''(w_1) = \Sigma_{11}+\gamma \mu_{d1}^2>0$ for all $w_1\in \mathbb{R}$. Therefore, the function $g_1(\cdot)$ is strictly convex on $\mathbb{R}$.  Also, it is clear that $g_2$ is convex on $\mathbb{R}$.  Therefore $g = g_1 + g_2$ is a strictly convex function on $\mathbb{R}$.

Combining the coordinate-wise strict convexity with the fact that  the non-differentiable part of the objective function is separable, Theorem 5.1 of \citet{Tseng01} guarantees that coordinate descent algorithms converge to coordinate-wise minima.  Moreover, since all directional derivatives exist, every coordinate-wise minimum is also a local minimum.  A similar study on the convergence of the coordinate descent algorithm can be found in \citet{BrehenyHuang2011}.

%We are aware of the coming Constrained LARS algorithm by Daubechies etc. We are eager to implement our problem with Constrained LARS once they released the package.
%
%We are now ready to state the \textit{Constrained Coordinate Descent}(CCD) algorithm for problem (\ref{eq13}).

In each coordinate update, the computational complexity is $\mathcal{O}(p)$. A complete cycle through all $p$ variables costs $\mathcal{O}(p^2)$ operations. From our experience, CCD converges quickly after a few cycles if ``warm start" is used for the initial solution. Let $C$ denote the average number of cycles until convergence for each $\lambda$.  Then our algorithm CCD enjoys computational complexity  $\mathcal{O}(CKp^2)$.  %This is    compared with  the Fisher discriminant, where matrix inversion alone costs at least $\mathcal{O}(p^{2.376})$ operations (the Coppersmith-Winograd algorithm), though we should emphasize here that our algorithm has no ambition to fully recover the Fisher discriminant (this task is infeasible anyway).
The D-ROAD can be similarly implemented by replacing the covariance matrix with its diagonal.

\section{Asymptotic Property}\label{sec::asymptoticproperty}
\subsection{Risk Approximation}
%We assume that the true number of non-zero coefficients are small.
%(e.g. suppose $n=20,p=10000$, $s_0=6$)

%Denote the parameter by $\btheta=(\bmu_1,\bmu_2,\bSigma)$. We restrict our attention to
%$$\Gamma=\{(\bmu_1, \bmu_2, \bSigma): \bmu_d^T\bSigma^{-1}\bmu_d\geq D_p,a_0\leq\lambda_{\mbox{min}}(\bSigma)\leq\lambda_{\mbox{max}}(\bSigma)\leq b_0, \},$$
%where $D_p$ is a positive constant only depending on $p$, and $a_0$, $b_0$ are positive constants.

Let $\hat\bw_{c}$ be a sample version of $\bw_c$ in (\ref{eq9}),
\begin{equation}\label{eq14}
\hat\bw_{c}\in\argmin_{\|\bw\|_1\leq c, \bw^T\hat\bmu_d=1}\bw^T\hat \bSigma\bw.
\end{equation}
The fact that $\hat\bSigma$ is only positive semi-definite leads to potential non-uniqueness of $\hat\bw_c$.
%Properties of $\hat\bw_{c}$ are important because they are all we can observe in the learning context.
Now, we have three different classifiers: $\delta_{\bw_\infty}=\mathbb{I}\{{\bw_\infty^T}(\bX-\bmu_a)>0\}$, $\delta_{\bw_c}=\mathbb{I}\{\bw_{c}^T(\bX-\bmu_a)>0\}$ and $\hat\delta_{\bw_c}=\mathbb{I}\{\hat\bw_{c}^T(\bX-\hat\bmu_a)>0\}$. The first two are oracle classifiers, requiring the knowledge of unknown parameters $\bmu_1$, $\bmu_2$ and $\bSigma$, while the third one is the feasible classifier, ROAD, based on the sample.  Their classification errors are given by (\ref{eq2}).  Explicitly, the error rates are respectively $W(\delta_{\bw_\infty})$ [see (\ref{eq4})], $W(\delta_{\bw_c})$, and $W(\hat{\delta}_{\bw_c})$.  By (\ref{eq2}), an obvious estimator of the misclassification rate of $\hat{\delta}_{\bw_c}$ is
\begin{equation} \label{eq15}
W_n(\hat\delta_{\bw_c})=1-\Phi\left(
\frac{\hat\bw_c^T\hat\bmu_d}{(\hat\bw_c^T\hat\bSigma\hat\bw_c)^{1/2}} \right).
\end{equation}
Two questions arise naturally:
\begin{enumerate}
   \item how close is $W(\hat{\delta}_{\bw_c})$, the misclassification error of
     $\hat{\delta}_{\bw_c}$, to that of its oracle $W(\delta_{\bw_c})$?
   \item does $W_n(\hat\delta_{\bw_c})$ estimate $W(\hat{\delta}_{\bw_c})$ well?
\end{enumerate}
Theorem \ref{thm1} addresses these two questions.  We introduce an intermediate optimization problem for convenience:
\[
\bw^{(1)}_c=\argmin_{\|\bw\|_1\leq c, \bw^T\hat\bmu_d=1}\bw^T\bSigma\bw.
\]

\begin{thm}\label{thm1}
Let $s_c=\|\bw_c\|_0$, $s_c^{(1)}=\|\bw_c^{(1)}\|_0$, and $\hat{s}_c=\|\hat{\bw}_c\|_0$. Assume that $\lambda_{\min}(\bSigma)\geq\sigma_0^2>0$, $\|\hat\bSigma-\bSigma\|_{\infty}=O_p(a_n)$ and $\|\hat\bmu_d-\bmu_d\|_{\infty}=O_p(a_n)$ for a given sequence $a_n \to 0$.  Then, we have
$$
   W(\hat\delta_{\bw_c})-W(\delta_{\bw_c})=O_p(d_n)\,,
$$
and
$$
    W_n(\hat\delta_{\bw_c})-W(\hat{\delta}_{\bw_c})=O_p(b_n)\,,
$$
where $b_n=\left(c^2\vee s_c \vee s_c^{(1)}\right)a_n$ and $d_n= b_n\vee (\hat{s}_ca_n)$.
\end{thm}

\begin{remark}
In Theorem~\ref{thm1}, $\|\cdot\|_{\infty}$ is the element wise super-norm.  When $\hat{\bSigma}$ is the sample covariance, by \cite{BickelLevina-04}, $\|\hat\bSigma-\bSigma\|_{\infty}=O_p(\sqrt{(\log p)/ {n}})$; hence we can take $a_n = \sqrt{(\log p) /{n}}$.  The first result in Theorem~\ref{thm1} shows the difference between the misclassification rate of $\hat{\delta}_{\bw_c}$ and its oracle version $\delta_{\bw_c}$; the second result says about the error in estimating the true misclassification rate of ROAD.
\end{remark}

\begin{remark}
In view of (\ref{eq2}), one intends to choose a $\bw$ that makes $\bw^T \bSigma \bw$ small and $\bw^T \bmu_d$ large.  A compromise of these dual objectives leads to a utility function
$$
    U(\bw) = -\bw^T \bSigma \bw + \xi \bmu_d^T\bw,
$$
as a proxy of the objective function (\ref{eq2}) for a fixed $\xi$.
For any $\xi > 0$, the optimal choice $\bw^{*}\in \mbox{\rm argmin } U(\bw)$ leads to the Fisher discriminant rule.  Consider also the regularized versions
$$
 \bw_c^* = \mbox{\rm argmin}_{\|\bw\|_1 \leq c} U(\bw), \quad \mbox{and} \quad
 \hat{\bw}_c^* = \mbox{\rm argmin}_{\|\bw\|_1 \leq c} \hat U(\bw),
$$
where $\hat{U}(\bw)$ is the utility function with $\bSigma$ and $\bmu_d$ estimated by
$\hat{\bSigma}$ and $\hat{\bmu}_d$.  Then, it is easy to see the following utility approximation: for any $\|\bw\|_1 \leq c$
$$
  |U(\bw) - \hat{U}(\bw)| \leq \| \hat{\bSigma} - \bSigma \|_\infty c^2 + \xi c \| \hat{\bmu}_d - {\bmu}_d \|_\infty
$$
and
$$
 |U(\hat{\bw}_c^*) - U(\bw_c^*)| \leq 2 \Bigl (\| \hat{\bSigma} - \bSigma \|_\infty c^2 + \xi c \| \hat{\bmu}_d - {\bmu}_d \|_\infty \Bigr ).
$$
\end{remark}

\begin{remark}
The most prominent technical challenge of our original problem (\ref{eq9})  is due to different dualities of penalization problems.  For the population version (\ref{eq9}), it can be reduced, by the Lagrange multiplier method, to the utility $U(\bw)$ optimization problem in Remark 2 with a given $\xi > 0$, while for the sample version (\ref{eq14}), it can be reduced to the utility $\hat{U}(\bw)$ optimization problem with a different $\hat{\xi}$. Therefore, the problem is not the same as the utility optimization problem in Remark 2:  $\hat{\xi}$ is hard to bound.  In fact, it is much harder and yields more complicated results.
\end{remark}

We now show how different the data projection direction in the regularized oracle can be from that in the Fisher discriminant.  To gain better insight, we reformulate the $L_1$ constraint problem as the following penalized version:
\begin{equation} \label{eq16}
   \bw^\lambda=\argmin_{\bw: \bmu_d^T\bw=1} \bw^T\bSigma\bw + \lambda\|\bw\|_1.
\end{equation}
The following characterizes its convergence to the Fisher discriminant weight $\bw_\infty$ as $\lambda \to 0$.

\begin{thm}\label{thm2}
Let $s$ be the size of the set $\{k:(\bSigma^{-1}\bmu_d)_k\neq0\}$.  Then,
we have
$$\|\bw^\lambda-\bw_\infty\|_2\leq \frac{\lambda\sqrt{s}}{\lambda_{\min}(\bSigma)},$$
where $\bw_\infty = \bSigma^{-1} \bmu_d / (\bmu_d^T \bSigma^{-1} \bmu_d)$ is the normalized Fisher discriminant, optimizing (\ref{eq16}) with $\lambda = 0$.
%$$|R(\bw^\lambda)-R(\bw_\infty)|\leq \frac{\lambda_{\max}(\bSigma)}{\lambda_{\min}(\bSigma)}\lambda\sqrt{s}\|\bw^{\lambda}+\bw_{\infty}\|_2$$ where $s=\|K\|$,  and $R(\bw)=$.
\end{thm}

\subsection{Screening-based ROAD (S-ROAD)}\label{subsec::screenedpenalized}
Following the idea of Sure Independence Screening in \citet{FanLv-08}, we pre-screen all the features before hitting the ROAD.  The advantage of this two-step procedure is that we have a control on the total number of features used in the final classification rule.  A popular method for independent feature selection is the two-sample $t$-test \citep{Tibshirani-02,FanFan-08}, which is a specific case of marginal screening in \citet{FanLv-08}.  The sure screening property of such a method was demonstrated in \cite{FanFan-08}, which selects consistently the features with different means in the same settings as ours.

Once the features are selected, we can hit the ROAD, producing the vanilla Screening-based Regularized Optimal Affine Discriminant (S-ROAD1):
\begin{enumerate}
  \item[(1)] Employ a screening method to get $k$ features.
  \item[(2)] Apply ROAD to the $k$ selected features.
\end{enumerate}

In the first step, we use the $t$-statistics as the screening criteria and determine a data-driven threshold. This idea is motivated by a FDR criterion for choosing marginal screening threshold in \cite{ZhaoLi10}.
%For a given $q\in [0,1]$, the screening approach will allow only $1-q$ proportion  of unimportant variables to enter the model when $\bX$ and $Y$ are not related (the null model).
A random permutation $\pi$ of $\{1, \cdots, n\}$ is used to decouple $\bX_i$ and $Y_i$ so that the resulting data $(\bX_{\pi(i)}, Y_i)$ follow a null model, by which we mean that features have no prediction power for the class label. More specifically, the screening step is carried out as follows:

\begin{enumerate}
  \item Calculate the $t$-statistic $t_j$ for each feature $j$, where $j=1,\cdots,p$.
  \item For the permuted data pairs $(\bX_{\pi(i)}, Y_{i})$, recalculate the $t$-statistic $t_j^*$, for $j=1,\cdots, p$. (Intuitively,  if $j$ is the index of an important feature, $|t_j|$ should be larger than most of $|t_j^*|$, because the random permutation is meant to eliminate the prediction power of features.)
  \item For $q\in [0,1]$, let $\omega_{(q)}$ be the $q^{th}$ quantile of $\{|t_j^*|, j=1,2,\cdots,p\}$. Then, the selected set $\mathcal{A}$ is defined as
$$\mathcal{A}=\{j||t_j|\geq \omega_{(q)}\}.$$
\end{enumerate}
The choice of threshold is made to retain the features whose $t$-statistics are significant in the  two sample t-test.
%However, there is one implicit assumption: the random permutation process leads to an approximately null model.  We do not intend to make analytical assessment on the approximation.  However, for a broad spectrum of data analysis in Section \ref{sec::numerical-study},  this procedure works well.  In our implementation, we make a very stringent choice by setting $q=1$.
Alternatively, if the user knows his $k$, (due to budget constraints, etc.), then he can just rank $|t_j|$'s and choose the threshold accordingly.

The S-ROAD1 tracks the performance of oracle procedures like sub-Fisher (10 features) in Figure 1.  The feature space gotten by step (1) can be expanded by including those features which are most correlated with what have already been selected.  This additional variant, S-ROAD2, aims at achieving the performance of sub-Fisher (20 features) type of procedure in Figure 1.

To elaborate on the theoretical properties of S-ROAD1, assume with no loss of generality that the first $k$ variables are selected in the screening step.  Denote by $\bSigma_k$ the upper left $k\times k$ block of $\bSigma$ and $\bmu_k$ the first $k$ coordinates of $\bmu_d$.   Let
\begin{align*}
\bbeta_c=\argmin_{\|\bbeta\|_1\leq c, \bbeta^T\bmu_k=1}\bbeta^T\bSigma_k\bbeta.
\end{align*}
The quantities $\hat{\bbeta}_c$ and $\bbeta_c^{(1)}$ are defined similarly to $\hat{\bw}_c$ and $\bw_c^{(1)}$ (defined right before Theorem 1). Then denote by $\by_c=(\bbeta_c^T,\bzero^T)^T$, $\hat{\by}_c=(\hat{\bbeta}_c^T, \bzero^T)^T$ and $\by_c^{(1)}=(\bw_c^{(1)},\bzero^T)^T$. The next two theorems can be verified along lines similar to Theorems 1 and 2. Hence, the proofs are omitted.

\begin{thm}  %Theorem 3
If $\|\hat{\bSigma}_k-\bSigma_k\|_{\infty}=O_p(\sqrt{\log k/n})$, $\|\hat{\bmu}_k-\bmu_k\|_{\infty}=O_p(\sqrt{\log k/n})$, and $\lambda_{\min}(\bSigma_k)\geq \delta_0>0$,  then we have
$$W(\hat{\delta}_{\by_c})-W(\delta_{\by_c})=O_p(e_n),$$
and
$$
W_n(\hat{\delta}_{\by_c})-W(\delta_{\by_c})=O_p(e_n),$$
where $e_n = (c^2\vee k)\sqrt{\frac{\log k}{n}}$.
\end{thm}

This result is cleaner than Theorem \ref{thm1}, as the rate does not involve $s_c$ and $\hat{s}_c$:  they are simply replaced by the upper bound $k$. Accurate bounds for $s_c$ and $\hat{s}_c$ are of interest for future exploration,  but they are beyond the scope of this paper.

\begin{thm}
Let $\by^{\lambda}_k=\argmin_{\by: \bmu_d^T\by=1, \by\in M_k}R(\by)+\lambda\|\by\|_1$ where $M_k$ is the subspace in $R^p$ with the last $p-k$ components being zero, and
$\by^0 =( (\bSigma^{-1}_k\bmu_k)^T/(\bmu_k^T\bSigma_k^{-1}\bmu_k),\bzero^T)^T$. Then we have
$$\|\by^{\lambda}_k-\by^0\|_2\leq \frac{\lambda\sqrt{k}}{\lambda_{\min}(\bSigma_k)}.$$
%and
%$$|R(\by^{\lambda}_k)-R(\by^0)|\leq \frac{\lambda_{\max}(\bSigma_k)\lambda^2 k}{\lambda^2_{\min}(\bSigma_k)}+\frac{2\lambda_{\max}(\bSigma_k)\lambda \sqrt{k}\|\bbeta^{0}\|_2}{\lambda_{\min}(\bSigma_k)}.$$
\end{thm}

%To be more specific, we propose the following two versions of Screening-based ROAD:
%\begin{itemize}
%  \item PATH1: perform a two-sample $t$-test on the data set and for some $q>0$, define $\mathcal{A}_q=\{j||t^*_j|\geq q\}$ be the set of features where the two-sample $t$-test statistic is larger than the threshold $q$. Then we can select the top $s_1$ features, where
%      $$s_1=\max(s_{\min}, |\mathcal{A}|).$$ and $s_{\min}$ is the minimum number of feature selected.
%      To sum up, PATH1 will select $\{X_{(1)},X_{(2)},\cdots,X_{(s_1)}\}$ as the features in the screening stage.
%  \item PATH2: for each of the $s_1$ features selected by PATH1, say feature $X_{(j)}$,  find the variable $X_{(j)}^*$ which has the highest correlation with $X_{(j)}$ among the remaining variables and add it into the selected set. As a result, PATH2 will select $\{X_{(1)},X_{(1)}^*,\cdots, X_{(s_1)},X_{(s_1)}^*\}$ as the important features in the screening stage, where will be $s_2=2s_1$ features.
%\end{itemize}
%In the numerical studies, an extensive comparison will be made on the vanilla ROAD and its screening-based versions PATH1 and PATH2.

\subsection{Continuous Piecewise Linear Solution Path}\label{sec::piecewiselinear}

We use the word ``linear" when referring to ``affine", in line with the \emph{status quo} in the statistical community.  Continuous piecewise linear paths are of much interest to statisticians, as the property reduces the computational complexity of solutions and justifies the linear interpolations of solutions at discrete points. Previous well known investigations include \citet{Efron-04} and \citet{Rosset&Zhu07}.  Our setup differs from others mainly in that in addition to a complexity penalty, there is also an affine constraint. Our proof calls in point set topology, and is purely geometrical, in a spirit very different from the existing ones. In particular, we stress that the continuity property is intuitively correct, but it is far from a trivial consequence of the assumptions.   The authors also believe that the claim holds true even if the $p-1$ dimensional affine subspace constraint is replaced by more generic ones, though the technicality of the proof must be more involved.
\begin{thm}\label{thm::piecewiselinear}
Let $\bmu_d\in\mathbb{R}^p$ be a constant, and $\bSigma$ be a
positive definite matrix of dimension $p\times p$. Let
$$\bw_c=\argmin_{\|\bw\|_1\leq c,
\bw^{T}\bmu_d=1}\bw^{T}\bSigma\bw,$$ then $\bw_c$ is a continuous piecewise linear
function in $c$.
\end{thm}

\begin{prop}
$W(\delta_{\bw_c})$ is a Lipschitz function in c.
\begin{proof} Recall that
$$W(\delta_{\bw_c})=1-\Phi\left(1/(R(\bw_c))^{1/2}\right).$$
By Theorem \ref{thm::piecewiselinear} and the fact that composition of Lipschitz functions is again Lipschitz, the conclusion holds.
\end{proof}
\end{prop}

\section{Numerical Investigation}\label{sec::numerical-study}

In this section, several simulation and real data studies are conducted. We compare ROAD and its variants S-ROAD1 (Screening-based ROAD version 1), S-ROAD2 (Screening-based ROAD version 2) and D-ROAD (Diagonal ROAD) with
NSC (Nearest Shrunken Centroid), SCRDA (Shrunken
Centroids Regularized Discriminant Analysis), FAIR (Feature Annealed Independence Rule), NB (Naive Bayes),  NFR (Naive Fisher Rule, which uses the generalized inverse of the sample covariance matrix), as well as the Oracle.

In all simulation studies, the number of variables is $p=1000$, and the sample size of the training and testing data is $n=300$ for each class. Each simulation is repeated 100 times to test the stability of the method.
%The distribution of the errors for the two classes is  multivariate normal  with zero mean and a common covariance matrix $\bSigma$.
Without loss of generality, the mean vector of the first class $\bmu_1$ is set to be $\bzero$. We use five-fold cross-validation to choose the penalty parameter $\lambda$.
%In the PATH1 and PATH2, we choose the threshold for the $t$-statistic to be $q=3$ and the minimum number of features included in the model to be $s_{\min}=20$.

\subsection{Equal Correlation Setting, Sparse Fixed Signal}\label{sec::eqcor-sparsefixed}

In this subsection, we consider the setting where $\Sigma_{i,i}=1$ for all $i=1,\cdots,p$ and $\Sigma_{i,j}=\rho$ for all
$i,j=1,\cdots,p$ and $i\neq j$, and take $\bmu_2$ to be a sparse vector: $\bmu_2=(\bone_{10}^T,\bzero_{990}^T)^T$, where $\bone_d$ is a length $d$ vector with all entries 1, $\bzero_d$ is a length $d$ vector with all entries 0, where the sparsity size is $s_0=10$.  Also, we fix $\gamma=10$ in \eqref{eq13} for this simulation.  Sensitivity of the performance due to the choice of $\gamma$ will be investigated in the next subsection.

\begin{figure}
\caption{Solution Path for ROAD (left panel) and D-ROAD (right panel). Equal correlation setting ($\rho=0.5$), Sparse Signal ($s_0=10$) as in Section \ref{sec::eqcor-sparsefixed}.}\label{figure:simu:equalcorde}
\begin{center}

\includegraphics[scale=0.32]{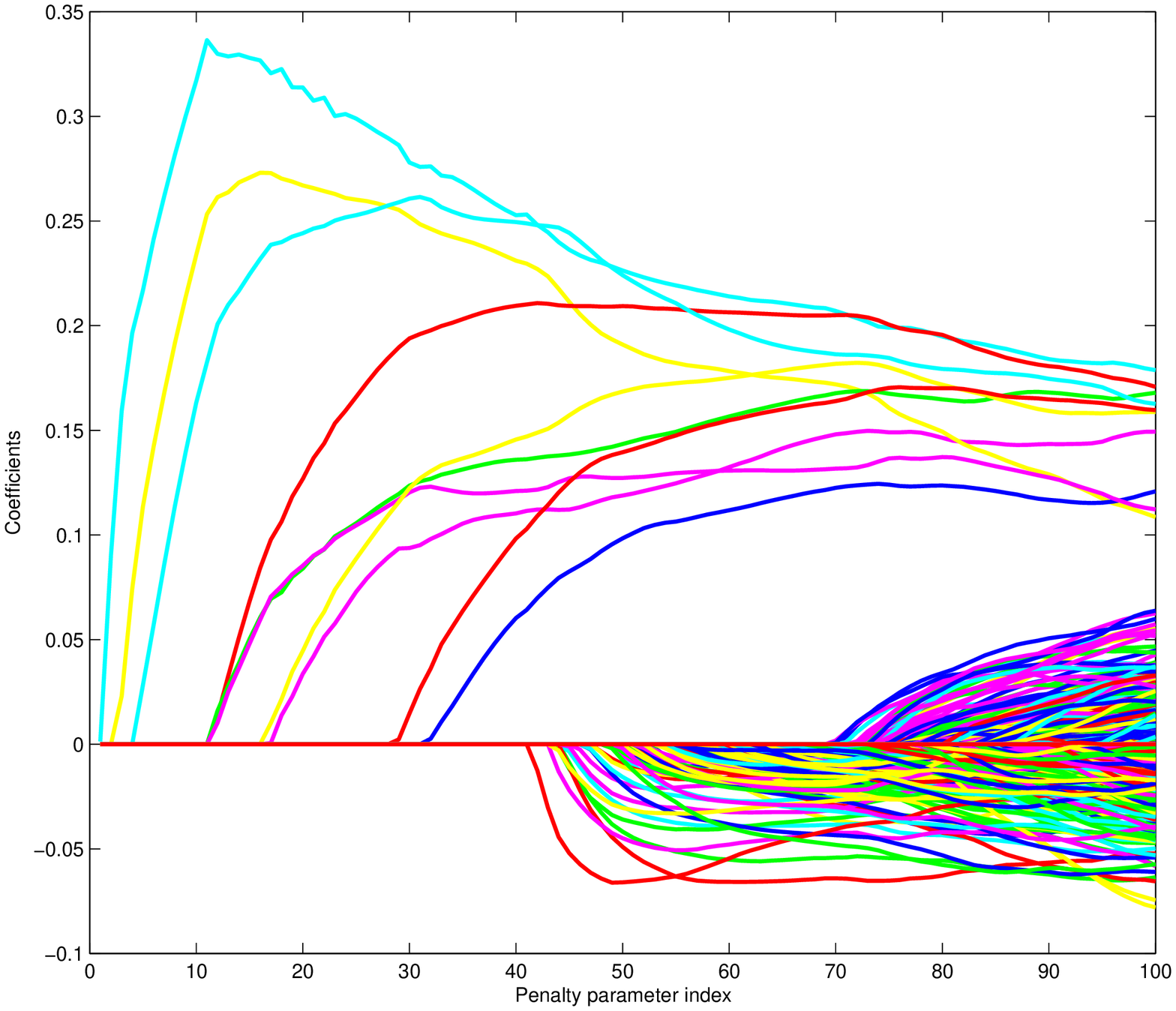}
\includegraphics[scale=0.32]{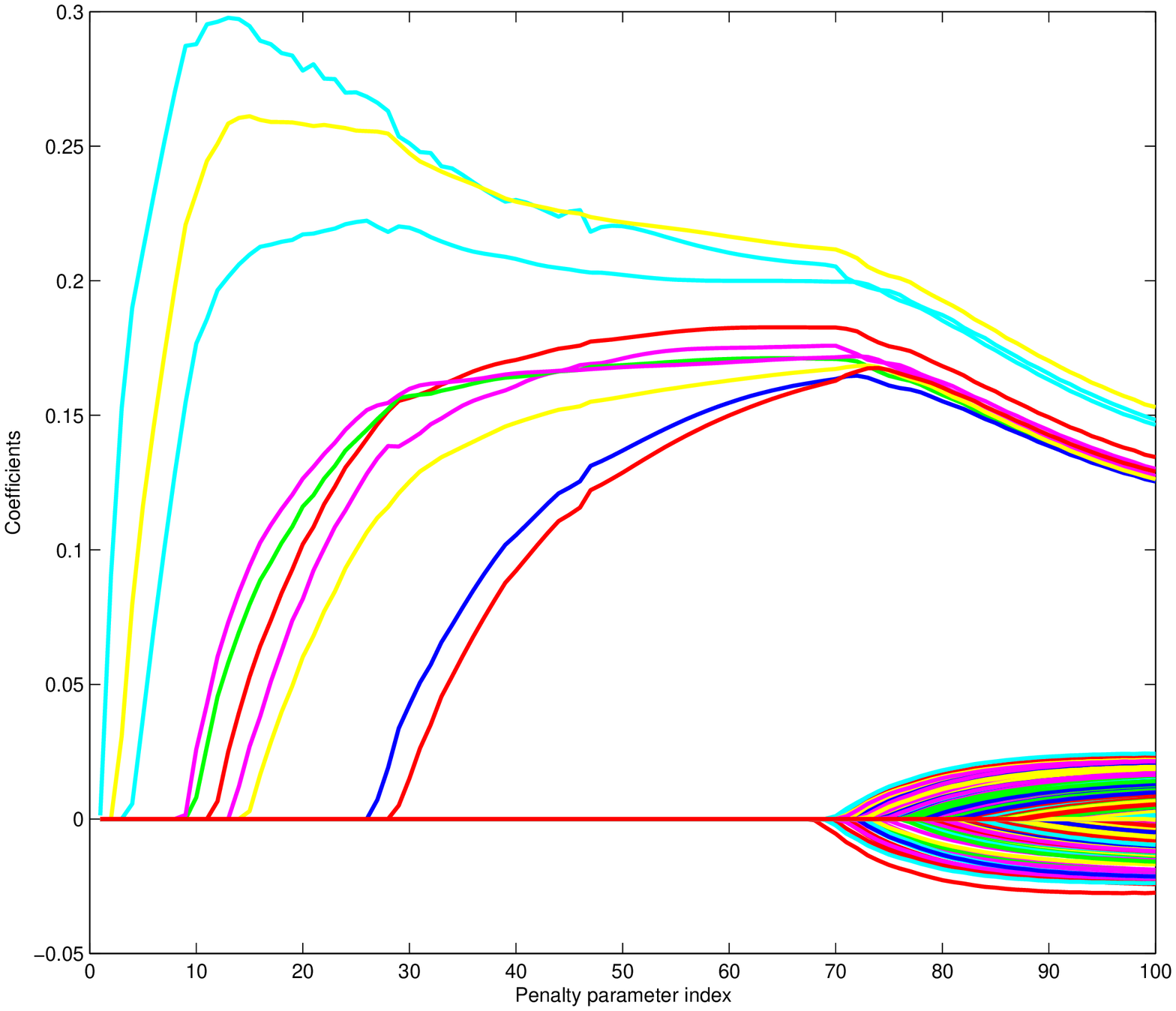}

\end{center}
\end{figure}

The solution paths for ROAD and D-ROAD of one realization are rendered in Figure \ref{figure:simu:equalcorde}. It is clear from the figure that, as the penalty parameter decreases (index increases), both ROAD and D-ROAD use more features. Also, the cutoff point for D-ROAD, where the number of features starts to increase dramatically, tends to come later than that for ROAD.

\begin{table}
\caption{Equal correlation setting, fixed signal: Median of the percentage for testing classification error and standard deviations (in parentheses). Signal all equal to 1. $s_0=10$.\label{tb-eqcor-fixed-median-error}}
\centering
\begin{tabular}{l|llllllllll}
\hline
$\rho$&ROAD&S-ROAD1&S-ROAD2&D-ROAD&SCRDA&NSC&FAIR&NB&Oracle\\
\hline
0&6.0(1.2)&6.0(1.1)&6.0(1.2)&5.7(1.1)&6.3(1.0)&5.9(1.0)&5.7(1.0)&11.2(1.4)&5.5(1.1)\\
0.1&6.3(2.5)&12.2(5.0)&8.8(2.4)&11.6(5.1)&10.3(1.4)&11.1(3.0)&12.4(1.4)&26.8(10.1)&5.0(0.9)\\
0.2&5.3(1.0)&16.0(6.3)&8.7(2.5)&16.1(7.5)&8.5(1.2)&14.5(4.3)&17.3(1.7)&34.8(11.6)&4.0(0.8)\\
0.3&4.2(0.9)&19.1(7.9)&7.8(2.6)&19.1(9.4)&6.6(1.1)&17.1(5.5)&20.8(1.7)&39.3(12.3)&3.2(0.7)\\
0.4&3.2(0.8)&22.8(9.4)&6.5(2.6)&22.2(9.9)&4.8(1.0)&20.5(6.1)&23.2(1.8)&41.6(11.3)&2.0(0.6)\\
0.5&2.0(0.6)&25.8(11.0)&4.8(1.4)&25.2(10.2)&2.9(0.7)&23.2(6.0)&25.3(1.6)&43.5(11.1)&1.3(0.5)\\
0.6&1.0(0.4)&18.3(12.4)&3.3(1.3)&28.1(10.3)&1.5(0.5)&25.8(5.7)&26.8(1.8)&44.4(12.1)&0.7(0.3)\\
0.7&0.3(0.2)&15.5(13.6)&1.7(1.0)&29.1(10.1)&0.5(0.3)&27.0(8.2)&28.2(2.0)&45.2(12.3)&0.2(0.2)\\
0.8&0.0(0.1)&5.0(14.0)&0.3(0.4)&29.5(9.9)&0.0(0.1)&28.3(8.7)&29.2(2.0)&46.2(10.3)&0.0(0.1)\\
0.9&0.0(0.0)&0.6(14.8)&0.0(0.1)&30.3(7.6)&0.0(0.2)&29.9(8.0)&30.2(1.9)&46.8(8.8)&0.0(0.0)\\
\hline
\end{tabular}
\end{table}

The simulation results for the pairwise correlations ranging from 0 to 0.9 are shown in Tables \ref{tb-eqcor-fixed-median-error} and \ref{tb-eqcor-fixed-nonzero}. We would like to mention that the results for NFR (Naive Fisher Rule) are not included in these (and the subsequent) tables because the test classification error is always  around 50\%, i.e., it is about the same as random guess. Also in the tables are the screening-based versions of the ROAD. S-ROAD1 refers to the vanilla version where we first apply the two-sample $t$-test to select any features with the corresponding $t$-test statistic with absolute value larger than the maximum absolute $t$-test statistic value calculated on the permuted data.  S-ROAD2 does the same except for each variable in S-ROAD1's pre-screened set, it adds an additional variable which is most correlated with that variable.  Figure \ref{fig3}, a graphical summary of Table~1,  presents the median test errors for different methods.  We can see from Table \ref{tb-eqcor-fixed-median-error} and Figure \ref{fig3} that the oracle classification error decreases as $\rho$ increases. This phenomenon is due to a similar reason to the two-dimensional showcase in the introduction. When $\rho$ goes to 1, all the variables contribute in the same way to boost the classification power. ROAD performs reasonably close to the Oracle, while working independence based method such as D-ROAD, NSC, FAIR and NB fail when $\rho$ is large.  The huge discrepancy shows the advantage of employing the correlation structure.   Since SCRDA also employ the correlation structure, it does not fail when $\rho$ is large. However, ROAD still outperforms SCRDA in all the correlation settings.  S-ROAD1 and S-ROAD2 both have misclassification rates similar to that of ROAD. %Moreover, the computation burden is significantly reduced by the pre-screening step, and the resulting classifier is also sparser than the vanilla ROAD.
It is worth to emphasize that the merits of the screening based ROADs mainly lie in the computation cost, which is reduced significantly by the pre-screening step.

The ROAD is a very robust estimator. It performs well even when all the variables are independent, in which case there could be a lot of noise for fitting the covariance matrix. Table \ref{tb-eqcor-fixed-median-error} indicates that ROAD has almost the same performance as D-ROAD, NSC and FAIR under the independence assumption, i.e.  $\rho=0$.  As $\rho$ increases, the edge of ROAD becomes more substantial. In general, the ROAD is recommended on the grounds that even with pairwise correlation of about 0.1 (which is quite common in microarray data as well as financial data), the gain is substantial.

Another interesting observation is that the D-ROAD performs similarly to NSC and FAIR in terms of classification error. An intuitive explanation is that they are all ``sparse" independence rules.   NSC uses  soft-thresholding on the standardized sample mean difference, and its equivalent LASSO derivation can be found in \citet{Wang&Zhu:nsclasso:2007}.  FAIR selects features with large marginal $t$-statistics in absolute values, while D-ROAD is another L1 penalized independence rule, whose implementation is different from NSC.

Table~2 summarizes the number of features selected by different classifiers.  Note that ROAD mimics Fisher discriminant coordinate $\bSigma^{-1}\bmu_d$, which has $p=1000$ nonzero entries under our simulated model.  Therefore, the large number of features selected by ROAD is not out of expectation.

\begin{figure}[t]
\caption{Median classification error  as a function of $\rho$ in the equi-correlation matrix.   Sparse $\bmu_d$ as in Section \ref{sec::eqcor-sparsefixed}.}\label{fig3}
\begin{center}
\includegraphics[scale=0.6]{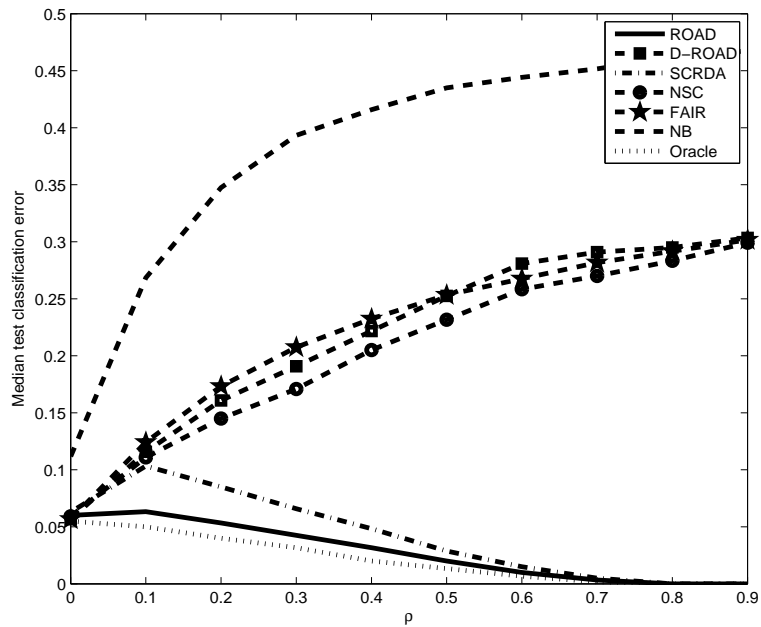}
\end{center}
\end{figure}

\begin{table}
\caption{Equal correlation setting, fixed signal: Median of number of nonzero coefficients and standard deviations (in parentheses). Signal all equal to 1. $s_0=10$.\label{tb-eqcor-fixed-nonzero}}
\centering
\begin{tabular}{l|lllllll}
\hline
$\rho$&ROAD&S-ROAD1&S-ROAD2&D-ROAD&SCRDA&NSC&FAIR\\
\hline
0&16.00(24.16)&10.00(1.31)&17.00(4.31)&29.50(58.54)&10.00(13.25)&10.00(44.86)&11.00(1.62)\\
0.1&117.50(30.50)&11.00(3.32)&21.00(4.15)&14.00(122.02)&1000.00(345.48)&35.50(117.32)&10.00(0.27)\\
0.2&130.50(33.33)&11.00(6.99)&22.00(8.98)&15.50(111.42)&1000.00(0.00)&95.00(120.17)&10.00(0.69)\\
0.3&136.50(36.16)&11.00(11.56)&22.00(10.38)&17.50(106.16)&1000.00(0.00)&103.50(117.52)&9.00(1.19)\\
0.4&135.00(34.43)&10.00(14.21)&22.00(17.07)&10.00(98.10)&1000.00(0.00)&70.00(131.65)&8.00(1.33)\\
0.5&138.50(38.17)&9.00(21.71)&22.00(21.56)&10.00(105.33)&1000.00(0.00)&65.00(137.97)&7.00(1.30)\\
0.6&148.00(49.74)&10.50(27.92)&22.00(31.88)&10.00(110.23)&1000.00(0.00)&38.00(141.91)&6.00(1.30)\\
0.7&170.50(52.29)&11.00(37.37)&22.00(41.76)&1.00(118.43)&1000.00(0.00)&27.50(140.10)&5.00(1.20)\\
0.8&203.00(27.72)&12.00(50.36)&24.00(59.23)&1.00(143.83)&1000.00(10.92)&15.00(157.98)&5.00(1.29)\\
0.9&151.50(8.02)&14.00(55.32)&28.00(50.45)&1.00(153.27)&1000.00(56.30)&14.00(225.38)&3.00(1.08)\\
\hline
\end{tabular}
\end{table}

\subsection{The Effect of $\gamma$}
Under the settings of the previous subsection, we look into the variation of  the ROAD performance as $\gamma$ changes. In Table \ref{tb-different-gamma}, the number of active variables varies; however, the median classification error remains about the same for a broad range of $\gamma$ values. The reason is that the cross validation step chooses the ``best" $\lambda$ according to a specific $\gamma$. Therefore, the final performance remains almost unchanged. Since our primary concern is the classification error, we fix $\gamma=10$ for simplicity in the subsequent simulations and in the real data analysis.
\begin{table}
\caption{Equal correlation setting; signals all equal to 1; $s_0=10$. Results for different $\gamma$.\label{tb-different-gamma}}
\centering
\begin{tabular}{l|llll}
\hline
&&$\rho=0$&$\rho=0.5$&$\rho=0.9$\\
\hline
\multirow{5}{2cm}{Median classification error (in percentage)}&ROAD$_{\gamma=0.01}$&5.8(1.2)&2.7(0.6)&0.2(0.2)\\
&ROAD$_{\gamma=0.1}$&6.0(1.2)&2.0(0.6)&0.2(0.1)\\
&ROAD$_{\gamma=1}$&6.0(1.3)&2.0(0.6)&0.0(0.1)\\
&ROAD$_{\gamma=10}$&6.0(1.2)&2.0(0.6)&0.0(0.0)\\
&ROAD$_{\gamma=100}$&6.2(1.2)&2.3(0.6)&0.0(0.1)\\
\hline
&&$\rho=0$&$\rho=0.5$&$\rho=0.9$\\
\hline
\multirow{5}{2cm}{Median number of nonzeros}&ROAD$_{\gamma=0.01}$&14.0(19.2)&129.5(42.5)&657.0(179.6)\\
&ROAD$_{\gamma=0.1}$&14.0(19.6)&137.0(37.6)&773.5(103.2)\\
&ROAD$_{\gamma=1}$&16.5(22.9)&139.0(37.9)&514.0(39.7)\\
&ROAD$_{\gamma=10}$&16.0(24.2)&138.5(38.2)&151.5(8.0)\\
&ROAD$_{\gamma=100}$&22.0(16.1)&114.5(9.4)&94.0(9.6)\\
\hline
\end{tabular}
\end{table}

\subsection{Block Diagonal Correlation Setting, Sparse Fixed Signal}\label{sec::block-diagonal-sparsefixed}
In this subsection, we follow the same setup as in Section \ref{sec::eqcor-sparsefixed} except that the covariance matrix $\bSigma$ is taken to be block diagonal. The first block is a $20\times20$ equi-correlated matrix and the second block is a $(p-20)\times(p-20)$ equi-correlated matrix, both with pairwise correlation $\rho$. In other words,
$\Sigma_{i,i}=1$ for all $i=1,\cdots,p$,  $\Sigma_{i,j}=\rho$ for all
$i,j=1,\cdots,20$ and $i\neq j$, $\Sigma_{i,j}=\rho$ for all $i,j=21,\cdots,p$ and $i\neq j$, and the rest elements are zeros. As before, we examine the performances of various estimators when $\rho$ varies. The percentage for testing  error and the number of selected features in the estimators are shown in Tables \ref{tb-block-diagonal-fixed-median-error} and \ref{tb-block-diagonal-fixed-nonzero}, respectively.

\begin{table}
\caption{Block diagonal correlation setting, sparse fixed signal: Median of the percentage for testing classification error and standard deviations (in parentheses). Signal all equal to 1. $s_0=10$.\label{tb-block-diagonal-fixed-median-error}}
\centering
\begin{tabular}{l|lllllllll}
\hline
$\rho$&ROAD&S-ROAD1&S-ROAD2&D-ROAD&SCRDA&NSC&FAIR&NB&Oracle\\
\hline
0&6.0(1.2)&6.0(1.1)&6.0(1.2)&5.7(1.1)&6.0(0.1)&5.5(0.3)&5.7(1.0)&11.2(1.4)&5.5(1.1)\\
0.1&10.8(3.6)&13.0(4.8)&10.3(3.0)&12.8(4.4)&13.0(0.3)&12.5(0.8)&12.7(1.5)&25.7(7.6)&8.8(1.2)\\
0.2&10.7(4.1)&18.0(5.7)&9.7(3.6)&17.7(5.9)&14.2(1.1)&17.2(0.4)&17.7(1.6)&34.4(7.9)&8.8(1.2)\\
0.3&9.5(3.8)&23.2(5.5)&8.8(4.0)&23.2(5.6)&12.7(0.9)&20.0(0.8)&20.4(1.6)&38.3(7.5)&7.7(1.0)\\
0.4&8.0(3.3)&29.7(4.2)&7.5(4.2)&29.3(4.1)&11.0(1.2)&23.8(1.3)&23.2(1.8)&41.0(6.9)&6.6(1.1)\\
0.5&6.2(2.6)&30.1(3.9)&5.7(0.9)&30.0(3.1)&8.7(0.4)&26.2(1.7)&25.1(1.7)&42.2(6.6)&5.0(1.0)\\
0.6&4.2(0.9)&30.3(4.2)&4.0(0.8)&30.3(2.2)&6.4(0.1)&26.5(1.2)&26.8(1.8)&43.6(7.0)&3.5(0.7)\\
0.7&2.3(0.7)&30.0(6.4)&2.2(0.7)&30.6(2.1)&2.5(0.7)&28.1(3.2)&28.2(2.0)&44.2(6.5)&1.8(0.6)\\
0.8&0.8(0.4)&29.8(9.8)&0.7(0.4)&30.6(2.1)&0.6(0.4)&29.2(1.6)&29.2(2.0)&44.8(5.7)&0.7(0.3)\\
0.9&0.0(0.1)&29.8(12.8)&0.0(0.1)&30.6(1.9)&0.2(0.2)&29.2(1.2)&30.2(1.9)&45.2(4.9)&0.0(0.1)\\
\hline
\end{tabular}
\end{table}

\begin{table}
\caption{Block diagonal correlation setting, fixed signal: Median of number of nonzero coefficients and standard deviations (in parentheses). Signal all equal to 1. $s_0=10$.\label{tb-block-diagonal-fixed-nonzero}}
\centering
\begin{tabular}{l|lllllll}
\hline
$\rho$&ROAD&S-ROAD1&S-ROAD2&D-ROAD&SCRDA&NSC&FAIR\\
\hline
0&16.00(24.16)&10.00(1.31)&17.00(4.31)&29.50(58.54)&10.00(1.15)&10.00(1.73)&11.00(1.62)\\
0.1&48.50(35.99)&10.00(2.73)&20.00(3.77)&14.00(26.73)&33.00(17.79)&65.00(38.84)&18.00(2.67)\\
0.2&48.00(31.48)&10.00(4.59)&20.00(5.84)&10.00(18.23)&38.00(117.54)&10.00(16.17)&18.00(2.77)\\
0.3&47.50(42.75)&9.00(5.28)&20.00(6.03)&10.00(11.80)&208.00(103.94)&10.00(13.58)&18.00(3.91)\\
0.4&40.50(32.42)&1.00(4.82)&20.00(10.08)&1.00(9.25)&27.00(90.95)&33.00(14.22)&17.00(5.43)\\
0.5&40.50(33.23)&1.00(4.88)&20.00(10.10)&1.00(8.51)&24.00(76.79)&10.00(1.15)&7.00(5.98)\\
0.6&39.50(30.03)&1.00(3.74)&20.00(14.53)&1.00(5.92)&127.50(6.36)&6.50(2.12)&6.00(5.98)\\
0.7&40.00(41.35)&1.00(4.71)&20.00(8.07)&1.00(2.49)&94.50(2.12)&9.50(0.71)&5.00(5.52)\\
0.8&55.00(58.67)&1.00(6.20)&20.00(18.32)&1.00(0.93)&58.00(2.83)&6.00(5.66)&5.00(4.84)\\
0.9&120.00(30.66)&1.00(21.29)&20.00(30.46)&1.00(0.35)&20.00(0.00)&8.00(2.83)&3.00(3.81)\\
\hline
\end{tabular}

\end{table}

In this block-diagonal setting, we have observed similar results to those in Section \ref{sec::eqcor-sparsefixed}: ROAD and S-ROAD2 perform significantly better than the other methods. One interesting phenomenon is that S-ROAD1 does not perform well when $\rho$ is large. The reason is that the current true model has 20 important features, and by looking only at marginal contribution, S-ROAD1 misses some important variables, as shown in Table \ref{tb-block-diagonal-fixed-median-error}. Indeed, because those features have no expressed mean differences, it does not fully take advantage of highly correlated features. In contrast, S-ROAD2 is able to pick up all the important variables,  takes advantage of correlation structure, and leads to a sparser model than the vanilla ROAD. In view of the results from this simulation setting and the previous one, we recommend S-ROAD2 over S-ROAD1.

\subsection{Block-Diagonal Negative Correlation Setting, Sparse Fixed Signal}

In this subsection, we again follow a similar setup as in Section \ref{sec::eqcor-sparsefixed}. Here, the covariance matrix $\bSigma$ is taken to be block diagonal with each block size equals to 10. Each block is an equi-correlated matrix with pairwise correlation $\rho=-0.1$. In other words,
$\bSigma=\mbox{diag}(\bSigma_0,\cdots,\bSigma_0)$, where $\bSigma_0$ is a $10\times 10$ equi-correlated matrix with correlation $-0.1$. Here, $\bmu_2=0.5\times(\bone_5^T, \bzero_5^T,\bone_5^T,\bzero_{985}^T)^T$ and the sparsity size is $s_0=10$.  As before, we examine the performances of various estimators when $\rho$ varies. The percentage for testing error and the number of selected features in the estimators are shown in Table \ref{tb-block-diagonal-neg-corr}.

\begin{table}

\caption{Block-Diagonal Negative Correlation Setting, Sparse Fixed Signal: Median error (in percentage) and number of nonzero coefficients with standard deviations in parentheses.\label{tb-block-diagonal-neg-corr}}
\centering
\resizebox{\textwidth}{!}{ %
\begin{tabular}{llllllllll}
\hline
&ROAD&S-ROAD1&S-ROAD2&D-ROAD&SCRDA&NSC&FAIR&NB&Oracle\\
\hline
error&7.3(3.4)&16.0(5.2)&12.7(3.4)&17.8(8.0)&18.5(1.1)&20.8(0.6)&24.8(2.1)&33.5(2.1)&3.2(0.7)\\
nonzero&168.00(47.59)&10.00(2.40)&20.00(3.58)&15.50(15.32)&24.00(0.58)&41.00(17.90)&59.00(4.27)&--&--\\
%nonzero&17.0(1.8)&28.0(4.5)&104.0(38.7)&5.5(7.6)&5.0(5.2)&1.0(0.4)&--&--&--\\
\hline
\end{tabular}}
\end{table}

\subsection{Random Correlation Setting, Double Exponential Signal}

To evaluate the stability of the ROAD, we take a random matrix $\bSigma$ as the correlation structure, and use a signal $\bmu$ whose nonzero entries come from a double exponential distribution. A random covariance matrix $\bSigma$ is generated as follows:
\begin{enumerate}
  \item For a given integer $m$ (here we take $m=10$), generate a $p\times m$ matrix $\bOmega$ where $\Omega_{i,j} \sim \mbox{Unif}(-1,1)$. Then the matrix $\bOmega\bOmega^T$ is positive semi-definite.
\item Denote $c_{\bOmega}=\min_i(\bOmega\bOmega^T)_{ii}$. Let $\bXi=\bOmega\bOmega^T+c_{\bOmega}\bI$, where $\bI$ is the identity matrix. It is clear that $\bXi$ is positive definite.
\item Normalize the matrix $\bXi$ to get $\bSigma$ whose diagonal elements are unity.
\end{enumerate}
For the signal, we take $\bmu$ to be a sparse vector with sparsity size $s=10$, and the nonzero elements are generated from the double exponential distribution with density function

$$f(x)= \exp(-2 |x|).$$

Table \ref{tb-randcor-de} summaries the results. It shows that even under random correlation setting and random signals, our procedure ROAD still outperforms other competing classification rules such as SCRDA, NSC and FAIR in terms of the classification error.

\begin{table}

\caption{Random correlation setting, double exponential signal: Median error (in percentage) and number of nonzero coefficients with standard deviations in parentheses.\label{tb-randcor-de}}
\centering
\resizebox{\textwidth}{!}{ %
\begin{tabular}{llllllllll}
\hline
&ROAD&S-ROAD1&S-ROAD2&D-ROAD&SCRDA&NSC&FAIR&NB&Oracle\\
\hline
error&2.0(0.6)&11.0(5.2)&5.8(3.9)&17.0(2.2)&5.2(1.1)&16.2(1.3)&17.0(1.6)&46.2(2.4)&1.3(0.5)\\
nonzero&83.00(39.54)&4.00(8.13)&9.00(10.69)&1.00(3.89)&1000.00(0.00)&4.00(0.58)&1.00(0.17)&--&--\\
\hline
\end{tabular}}
\end{table}

\subsection{Real Data}\label{sec::realdata}
Though the ROAD seems to perform best in a broad spectrum of idealized experiments, it has to be tested against reality.  We now evaluate the performance of our newly proposed estimator on three popular gene expression data sets:  ``Leukemia" \citep{Golub-1999},  ``Lung Cancer" \citep{Gordon-2002}, and ``Neuroblastoma data set" \citep{Oberthuer06}. The first two data sets come with predetermined,
separate training and test sets of data vectors. The Leukemia
data set contains $p = 7,129$ genes for $n_1 = 27$ acute lymphoblastic
leukemia (ALL) and $n_2$ = 11 acute myeloid leukemia
(AML) vectors in the training set. The test set includes 20 ALL
and 14 AML vectors. The Lung Cancer data set contains
$p = 12,533$ genes for $n_1 = 16$ adenocarcinoma (ADCA) and
$n_2 = 16$ mesothelioma training vectors, along with 134 ADCA
and 15 mesothelioma test vectors. The Neuroblastoma data set, obtained via the
MicroArray Quality Control phase-II (MAQC-II) project, consists of gene expression profiles
for $p=10,707$ genes from 251 patients of the German Neuroblastoma Trials NB90-NB2004,
diagnosed between 1989 and 2004. We analyzed the gene expression data with the 3-year event-free survival (3-year
EFS), which indicates whether a patient survived 3 years after the
diagnosis of neuroblastoma. There are
239 subjects with the 3-year EFS information available (49 positives and 190 negatives). We randomly select
83 subjects (19 positives and 64 negatives, which are about one third of the total subjects) as the training set and the
rest as the test set.
The readers can find more details about the data sets in the original papers.

Following \citet{Dudoit-2002} and \citet{FanFan-08}, we standardized each sample to zero mean and unit variance. The classification results for ROAD, S-ROAD1, S-ROAD2, SCRDA, FAIR, NSC and NB are shown in Tables \ref{tb-leukemia}, \ref{tb-lung} and \ref{tb-nb}. For the leukemia and lung cancer data, ROAD performs the best in terms of classification error. For the neuroblastoma data, NB performs best, however, it makes use of all 10,707 genes, which is not very desirable. In contrast, ROAD has a competitive performance in terms of classification error and it only selects 33 genes.  Although SCRDA has a close performance, the number of selected variables varies a lot for the three data set (264, 2410, 1). Overall, ROAD is a robust classification tool for high-dimensional data.

\begin{table}
\caption{Classification error and number of selected genes by various methods of leukemia data. Training and testing samples are of sizes 38 and 34, respectively.\label{tb-leukemia}}
\centering
\begin{tabular}{l|lllllll}\hline
%&ROAD&PATH1&PATH2&FAIR&NSC&IR\\\hline
&ROAD&S-ROAD1&S-ROAD2&SCRDA&FAIR&NSC&NB\\\hline
Training Error & 0 & 0 & 0 & 1 & 1 & 1 & 0 \\
  Testing Error & 1 & 3 & 1 & 2 & 1 & 3 & 5 \\
  No. of selected genes & 40 & 49 & 66 & 264 & 11 & 24 & 7129 \\
\hline
\end{tabular}
\end{table}

\begin{table}
\caption{Classification error and number of selected genes by various methods of lung cancer data. Training and testing samples are of sizes 32 and 149, respectively.\label{tb-lung}}
\centering
\begin{tabular}{l|lllllll}\hline
%&ROAD&PATH1&PATH2&FAIR&NSC&IR\\\hline
&ROAD&S-ROAD1&S-ROAD2&SCRDA&FAIR&NSC&NB\\\hline
Training Error & 1 & 1 & 1 & 0 & 0 & 0 & 6 \\
  Testing Error & 1 & 4 & 1 & 3 & 7 & 10 & 36 \\
  No. of selected genes & 52 & 56 & 54 & 2410 & 31 & 38 & 12533 \\
\hline
\end{tabular}
\end{table}

\begin{table}
\caption{Classification error and number of selected genes by various methods of neuroblastoma
data. Training and testing samples are of sizes 83 and 163, respectively.\label{tb-nb}}
\centering
\begin{tabular}{l|lllllll}\hline
%&ROAD&PATH1&PATH2&FAIR&NSC&IR\\\hline
&ROAD&S-ROAD1&S-ROAD2&SCRDA&FAIR&NSC&NB\\\hline
Training Error & 3 & 22 & 14 & 16 & 15 & 16 & 14 \\
  Testing Error & 33 & 47 & 37 & 37 & 44 & 35 & 32 \\
  No. of selected genes & 33 & 1 & 9 & 1 & 18 & 41 & 10707 \\
\hline
\end{tabular}
\end{table}

\section{Discussion}\label{sec::summary}
With a simple two-class gaussian model, we explored the bright side of using correlation structure for high dimensional classification.   Targeting directly on the classification error, ROAD employs un-regularized pooled sample covariance matrix and sample mean difference vector without suffering from curse of dimensionality and noise accumulation.  The sparsity of chosen features is evident in simulations and real data analysis; however, we have not discovered intuitively good conditions on $\bSigma$ and $\bmu_d$, such that a certain desirable sparsity pattern of $\hat{\bw}_c$ follows. We resolve a part of the problem by introducing screening-based variants of ROAD, but the precise control of the sparsity size  % and counterparts to sparse oracle inequalities in regression setup
is worth for further investigation.  Furthermore, we can explore the conditions for the model selection consistency.

In this paper, we have restricted ourselves to the linear rules. They can be easily extended to nonlinear discriminants via transformations such as low order polynomials or spline basis functions.
One may also use the popular ``kernel tricks" in the machine learning community.  See, for example, \citet{HTF09} for more details. After the features are transformed, we can hit the ROAD. %We have also related the ROAD to a constrained LASSO.  In view of this connection, ROAD can be generalized by relaxing the loss function.
One essential technical challenge of the current paper is rooted in a stochastic linear constraint. The precise role of this constraint has not been completely pinned down.   In the following, a preliminary proposal is provided for extending ROAD to multi-class settings.

\subsection{Extension to Multi-Class}\label{sec:multi class}
In this section, we outline an extension of ROAD to multi-class classification problems. Suppose that there are $K$ classes, and for $j=1,\cdots, K$, the $j$th class has mean $\bmu_j$ and common covariance $\bSigma$. Denote the overall mean of features by $\bmu_a = K^{-1} \sum_{j=1}^K\bmu_j$.  Fisher's reduced rank approach to multi-class classification is a minimum distance classifier in some lower dimensional projection space.  The first step is to find $s\leq K - 1$ discriminant coordinates ($\bw_1^*$, $\cdots$, $\bw_s^*$) that separate the population centroids $\{\bmu_j\}_{j=1}^K$ the most in the projected space $\mathcal{S}=\text{span}\{\bw_1^*$, $\cdots$, $\bw_s^*\}$.  Then the population centroids $\bmu_j$'s and new observation $\bX$ are both projected onto $\mathcal{S}$. The observation $\bX$ will be assigned to the class whose projected centroid is closest to the projection of $\bX$ onto $\mathcal{S}$.  Note that it is usually not necessary to compute all $K-1$ discriminant coordinates whose span is that of all $K$ population centroids; the process can stop as long as the projected population centroids are well spread out in $\mathcal{S}$.

We adopt the above procedure for multi-class classification.  However, the large-$p$-small-$n$ scenario demands regularization in selecting discriminant coordinates. Indeed, in the Fisher's proposal the first discriminant coordinate $\bw_1^*$ is the solution of
\begin{eqnarray}\label{eq:Rayleigh quotient}
\max_{\bw}\frac{\bw^T \bB\bw}{\bw^T\bSigma\bw}\,,
\end{eqnarray}
where $\bB=\bPsi^T\bPsi$, and the $j$th column of $\bPsi^T$ is $(\bmu_j-\bmu_a)$. Note that a multiple of $\bB$ is the between-class variance matrix.  The second discriminant coordinate $\bw_2^*$ is the maximizer of $\bw^T\bB\bw/ (\bw^T\bSigma\bw)$ with constraint $\bw_1^{*T}\bSigma\bw=0$, and the subsequent discriminant coordinates are determined analogously.

Since solving \eqref{eq:Rayleigh quotient} is the same as looking for the eigenvector  of $\bSigma^{-1/2}\bB \bSigma^{-1/2}$ corresponding to the largest eigenvalue, diverging spectrum and noise accumulation have to be considered when we work on the sample.   To address these issues, we regularize $\bw$ as in the binary case,
\begin{eqnarray}\label{eq:multiclass}
\min_{\|\bw\|_1\leq c, \bw^T \bB\bw=1}\bw^T\bSigma\bw ,
\end{eqnarray}
whose solution is the first regularized discriminant coordinate $\bar{\bw}^*_1$. Here, equation \eqref{eq:multiclass} is related to the null space method in \citep{Krzanowski-1995}.
The second regularized discriminant coordinate is obtained by solving \eqref{eq:multiclass} with additional constraint $\bar{\bw}^{*T}_1\bSigma \bw=0$. Other regularized discriminant coordinates can be found similarly.  With these $s$ ($\leq K-1$) regularized discriminant coordinates, the classifier is now based on the minimum distance to the projected centroids in the $s$-dimensional space spanned by $\{\bar{\bw}^*_j\}_{j=1}^s$.

The implementation and theoretical properties for multi-class ROAD are interesting topics for future research.
%\textcolor{red}{Remark:How to implement this additional linear constraint is not clear.  If we put it up as a quadratic constraint again, then there will be one more parameter to tune.  Or use something like SMO.}

%The theoretical properties for binary case is not immediately extendable to multi-class, but such extension is an interesting topic for future research.

\section*{Acknowledgements}
The authors thank the Editor, the Associate Editor and
two referees, whose comments have greatly improved the scope and presentation
of the paper. The financial support from
NSF grant DMS-0704337 and NIH Grant R01-GM072611 is
greatly acknowledged.
\appendix
\section{Proofs}
\subsection{Proof of Theorem \ref{thm1}}  We now show first part of the theorem. Let $f_0(\bw)=\bw^T\bmu_d/(\bw^T\bSigma\bw)^{1/2}$, $f_1(\bw)=\bw^T\hat\bmu_d/(\bw^T\bSigma\bw)^{1/2}$,
and $f_2(\bw)=\bw^T\hat\bmu_d/(\bw^T\hat\bSigma\bw)^{1/2}$.  Then, it follows easily that
\begin{align*}
 |f_0(\bw_c)-f_2(\hat\bw_c)|&\leq \Lambda_1 + \Lambda_2,
\end{align*}
where $\Lambda_1=|f_0(\bw_c)-f_1(\bw^{(1)}_c)|$ and $\Lambda_2=|f_1(\bw^{(1)}_c)-f_2(\hat\bw_c)|$.
We now bound both terms separately in the following two steps.

\textbf{Step 1(bound $\Lambda_1$):}
For any $\bw$, we have
\begin{align}\label{eq17}
 |f_0(\bw)-f_1(\bw)|&\leq |\frac{\bw^T\mu_d}{(\bw^T\bSigma\bw)^{1/2}}-\frac{\bw^T\hat\bmu_d}{(\bw^T\bSigma\bw)^{1/2}}|\nonumber\\
&\leq \frac{\|\bw\|_1\|\hat\bmu_d-\bmu_d\|_{\infty}}{\|\bw\|_2\lambda^{1/2}_{\min}(\bSigma)}\nonumber\\
&\leq \sqrt{\|\bw\|_0}\frac{\|\hat\bmu_d-\bmu_d\|_{\infty}}{\sigma_0}\nonumber\\
&= \sqrt{\|\bw\|_0}O_p(a_n).
%&\leq \mbox{\textcolor{red}{GOT STUCK HERE}}
 %\sigma_0^{-1/2}O_p(a_n)\nonumber\\
%&= O_p(a_n)
\end{align}
Since $\bw^{(1)}_c$ maximizes $f_1(\cdot)$, it follows that
\begin{align}
  f_0(\bw_c)-f_1(\bw^{(1)}_c) &=f_0(\bw_c)-f_1(\bw_c)+[f_1(\bw_c)-f_1(\bw^{(1)}_c)]\nonumber\\
&\leq f_0(\bw_c)-f_1(\bw_c), \label{eq18}
\end{align}
and similarly noticing $w_c$ maximizing $f_0(\cdot)$, we have
\begin{align}
  f_1(\bw^{(1)}_c)-f_0(\bw_c) &=f_1(\bw^{(1)}_c)-f_0(\bw^{(1)}_c)+[f_0(\bw^{(1)}_c)-f_0(\bw_c)]\nonumber\\
&\leq f_1(\bw^{(1)}_c)-f_0(\bw^{(1)}_c). \label{eq19}
\end{align}
Combining the results of (\ref{eq18}) and (\ref{eq19}) and using (\ref{eq17}), we conclude that
\begin{align*}
  \Lambda_1 = |f_0(\bw_c)-f_1(\bw^{(1)}_c)| = O_p\left((s_c\vee s_c^{(1)})a_n\right).
\end{align*}
By the Lipschitz property of $\Phi$,
$$
       |\Phi(f_1(\bw^{(1)}_c))-\Phi(f_0(\bw_c))|=O_p\left((s_c\vee s_c^{(1)})a_n\right).
$$

\textbf{Step 2(bound $\Lambda_2$):}
Note that $\bw^{(1)}_c$ and $\hat \bw_c$ both are in the set $\{\bw:  \bw^T \bmu_d = 1, \| \bw \|_1 \leq 1\}$.  Therefore,  by definition of minimizers, we have
$$
 {\bw^{(1)}_c}^T\bSigma\bw^{(1)}_c -\hat\bw_c^T\bSigma\hat\bw_c \leq 0,
 \mbox{ and } {\hat \bw_c}^T \hat \bSigma \hat\bw_c - {\bw^{(1)}_c}^T
 \hat \bSigma\bw^{(1)}_c\leq 0.
$$
Consequently,
\begin{align}\label{eq20}
{\bw^{(1)}_c}^T\bSigma\bw^{(1)}_c -\hat\bw_c^T\hat\bSigma\hat\bw_c&=[{\bw^{(1)}_c}^T\bSigma\bw^{(1)}_c -\hat\bw_c^T\bSigma\hat\bw_c]+\hat\bw_c^T\bSigma\hat\bw_c-\hat\bw_c^T\hat\bSigma\hat\bw_c\nonumber\\
&\leq \hat\bw_c^T(\bSigma-\hat\bSigma)\hat\bw_c\nonumber\\
&\leq \|\bSigma-\hat\bSigma\|_{\infty}\|\hat{\bw}_c\|^2_1\nonumber\\
&\leq  c^2\|\bSigma-\hat\bSigma\|_{\infty} \nonumber\\
&= O_p(a_nc^2).
\end{align}
By the same argument, we also have
\begin{align}\label{eq21}
\hat\bw_c^T\hat\bSigma\hat\bw_c - {\bw^{(1)}_c}^T\bSigma\bw^{(1)}_c &=[\hat\bw_c^T\hat\bSigma\hat\bw_c -{\bw^{(1)}_c}^T\hat\bSigma\bw^{(1)}_c]+{\bw^{(1)}_c}^T\hat\bSigma\bw^{(1)}_c- {\bw^{(1)}_c}^T\bSigma\bw^{(1)}_c\nonumber\\
&\leq {\bw^{(1)}_c}^T(\hat\bSigma-\bSigma){\bw^{(1)}_c}\nonumber\\
%&\leq \|\bSigma-\hat\bSigma\|_{\infty} \|\bw^{(1)}_c\|^2_1\nonumber\\
&\leq  c^2\|\bSigma-\hat\bSigma\|_{\infty}\nonumber\\
&= O_p(a_nc^2).
\end{align}
Combination of (\ref{eq20}) and (\ref{eq21}) leads to
$$|\hat\bw_c^T\hat\bSigma\hat\bw_c - {\bw^{(1)}_c}^T\bSigma\bw^{(1)}_c|=O_p(a_nc^2).$$

%Since $\|\hat\bmu_d-\bmu_d\|_{\infty}=O_p(a_n)$, we have  $\max_i{|\hat\mu_{d,i}|}\stackrel{p}{\leq}(2\max_i{|\mu_{d,i}|})$.

%$${\bw^{(1)}_c}^T\bSigma\bw^{(1)}_c\geq \lambda_0\|\bw^{(1)}_c\|_1^2\geq \frac{\lambda_0}{ \max_i^2{|\hat\mu_{d,i}|}}\stackrel{p}{\geq} \frac{\lambda_0}{4 \max_i^2{|\mu_{d,i}|}}\stackrel{p}{\geq}0,$$
%By the Lipschitz property of function $x^{-1/2}$ when $x\geq C>0$ where $C$ is a generic positive constant, we have
%\begin{align}
%  \Lambda_2&=|f_1(\bw^{(1)}_c)-f_2(\hat\bw_c)|\nonumber\\
%  &=|(\hat\bw_c^T\hat\bSigma\hat\bw_c)^{-1/2} - ({\bw^{(1)}_c}^T\bSigma\bw^{(1)}_c)^{-1/2}|\nonumber\\
%&\stackrel{p}{\leq} C(\hat\bw_c^T\hat\bSigma\hat\bw_c - {\bw^{(1)}_c}^T\bSigma\bw^{(1)}_c)\nonumber\\
%&\leq O_p(a_n)
%\end{align}
%
%Combining the bounds for $\Lambda_1$ and $\Lambda_2$,

Let $g(x) = \Phi(x^{-1/2})$.  The function $g$ is Lipschitz on $(0, \infty)$, as $g'(x)$ is bounded on $(0, \infty)$.  Hence, $|\Phi(f_2(\hat\bw_c))-\Phi(f_0(\bw^{(1)}_c))|=O_p(a_nc^2)$.  Thus,
\begin{align*}
|W_n(\hat\delta_{\bw_c},\btheta)-W(\delta_{\bw_c},\btheta)|
%& = |\Phi(f_2(\hat\bw_c))-\Phi(f_0(\bw_c))|\\
&\leq |\Phi(f_2(\hat\bw_c))-\Phi(f_0(\bw^{(1)}_c))|+|\Phi(f_1(\hat\bw^{(1)}_c))-\Phi(f_0(\bw_c))|\\
& = O_p\left((s_c\vee s_c^{(1)})a_n\right)+O_p(a_nc^2)\\
& = O_p(b_n).
\end{align*}

We now prove the second result of the Theorem.
Since $|\hat\bw_c^T\bSigma\hat\bw_c-\hat\bw_c^T\hat\bSigma\hat\bw_c|=O_p(a_nc^2)$, we have
\begin{align}\label{eq:f1f2diff}
|\Phi(f_1(\hat\bw_c))-\Phi(f_2(\hat\bw_c))|=O_p(a_nc^2).
\end{align}
By (\ref{eq17}), (\ref{eq:f1f2diff}), and the first part of the Theorem, we have
\begin{align*}
&|W(\hat\delta_{\bw_c},\btheta)-W(\delta_{\bw_c},\btheta)|\\
=&  |\Phi(f_0(\hat\bw_c))-\Phi(f_0(\bw_c))|\\
\leq& |\Phi(f_0(\hat\bw_c))-\Phi(f_1(\hat\bw_c))|+|\Phi(f_1(\hat\bw_c))-\Phi(f_2(\hat\bw_c))|+|\Phi(f_2(\hat\bw_c))-\Phi(f_0(\bw_c))|\\
=&   O_p(\hat{s}_ca_n)+O_p(a_nc^2)+O_p(b_n)\\
=&O_p(d_n).
\end{align*}
This completes the proof of Theorem.

\subsection{Proof of Theorem \ref{thm2}}
Let $\bw^\lambda=\bw_\infty + \bgamma^{\lambda}$. Then, from the definition of $\bw^\lambda$, we have
\begin{align}
\bgamma^{\lambda}
&=\argmin_{\bmu_d^T\bw_\infty+\bmu_d^T\bgamma=1}R(\bw_\infty+\bgamma)
  +\lambda\|\bw_\infty+\bgamma\|_1 \nonumber \\
& = \argmin_{\bmu_d^T\bgamma=0}f(\bgamma), \label{eq23}
\end{align}
where $f(\bgamma) =  R(\bgamma) + \lambda\Sigma_{k\in K^{c}}|\bgamma_k|+\lambda\Sigma_{k\in K}\left(|\bw_{\infty}^k+\bgamma_k|-|\bw_{\infty}^k|\right)$.
In the last statement, we used the fact that
$$
   \bw_\infty^T \bSigma \bgamma = \bmu_d^T \bgamma / (\bmu_d^T \bSigma^{-1} \bmu_d) = 0.
$$
We write $\bgamma$ for $\bgamma^{\lambda}$ for short in what follows.

By (\ref{eq23}), we have $f(\bgamma)\leq f(\bzero) =0$.  This implies that
\begin{eqnarray*}
R(\bgamma) & \leq &  \lambda\Sigma_{k\in K}\left(|\bw_{\infty}^k| - |\bw_{\infty}^k+\bgamma_k|\right)
 \leq  \lambda \Sigma_{k\in K}   |\bgamma_k| \leq \lambda \sqrt{s} \| \bgamma \|_2.
\end{eqnarray*}

%Let $\bgamma(K)$ be the vector with coefficients $\bgamma_k(K)=\bgamma_k\cdot\mathbb{I}\{k\in K\}$ and $\bgamma_k(K^c)=\bgamma_k \cdot\mathbb{I}\{k\not\in K\}$. Then $\bgamma=\bgamma(K)+\bgamma(K^c)$. We observe that $f(\bgamma)\leq f(\bzero) =0$.

%\begin{align*}
%\|\gamma(K^c)\|_1
%&= \Sigma_{k\in K^c}|\gamma_k|\\
%&\leq \left|\Sigma_{k\in K}|\bw_\infty^k+\gamma_k|-|\bw_\infty^k| \right|\\
%&\leq \|\gamma(K)\|_1
%\end{align*}
%As $\|\bgamma(K)\|_0\leq s$, it follows that $\|\bgamma(K)\|_1\leq\sqrt{s}\|\bgamma(K)\|_2\leq\sqrt{s}\|\bgamma\|_2$, which implies $R(\bgamma)\leq \lambda\sqrt{s}\|\bgamma\|_2$.

On the other hand, $R(\bgamma)\geq \lambda_{\min}(\bSigma)\|\bgamma\|_2^2$.
Bringing the upper and lower bound of $R(\bgamma)$ together, we conclude that $$\|\bgamma\|_2\leq\frac{\lambda\sqrt{s}}{\lambda_{\min}(\bSigma)}.$$
The proof is now complete.

%Denote $\bv^\lambda := \bSigma^{1/2}\bw^\lambda$ and $\bv_{\infty} := \bSigma^{1/2}\bv_{\infty}$, we have
%\begin{align*}
%|R(\bw^{\lambda})-R(\bw_{\infty})|
%&=|\langle \bv^{\lambda}+\bw_{\infty},\bv^{\lambda}-\bw_{\infty}\rangle|\\
%&\leq \|\bv^{\lambda}+\bv_{\infty}\|_2\cdot\|\bv^{\lambda}-\bv_{\infty}\|_2\\
%&= \|\bSigma^{1/2}(\bw^\lambda+\bw_\infty)\|_2\cdot\|\bSigma^{1/2}(\bw^{\lambda}-\bw_{\infty})\|_2\\
%&\leq \frac{\lambda_{\max}(\bSigma)}{\lambda_{\min}(\bSigma)}\lambda\sqrt{s}\|\bw^{\lambda}+\bw_{\infty}\|_2\\
%\end{align*}

\subsection{Proof of Theorem \ref{thm::piecewiselinear}}
By the positive definiteness of $\bSigma$, $\bSigma^{-1}$ and $\bSigma^{-\frac{1}{2}}$ are also positive definite.   Let
$\bv=\bSigma^{1/2}\bw$, then the transformation $\bv \mapsto \bw$ is linear.  Define
$$\bv_c=\argmin_{\|\bSigma^{-1/2}\bv\|_1\leq c,
\bv^{T}\bar{\bmu}_{d}=1}\bv^{T}\bv,$$ where
$\bar{\bmu}_{d}=\bSigma^{-1/2}\bmu_d$.  It is enough to show that
$\bv_c$ is piecewise linear in $c$.

Let $\Omega_c=\{\bv:\|\bSigma^{-1/2}\bv\|_1\leq c\}$ and $S=\{\bv:
\bv^{T}\bar{\bmu}_{d}=1\}$. When $c$ is small, the solution set is
$\emptyset$; when $c$ is large, the constraint $\Omega_c$ is inactive.  Denote by ``a" the smallest ``c" such that $\Omega_c\bigcap S\neq \emptyset$, and by ``b" the smallest such that $\bv_c$ are the same for all $c\geq b$.
Hence we are interested in $c\in[a,b]$, when changes in $c$ actually
affects the solution.

Let $P$ be the projection of the origin $O$ onto the hyperplane $S$
in the $p$ dimensional space. Let $$\mathcal{F}_c=\{S_{1,c}^0,
\cdots, S_{j_0,c}^0;S_{1,c}^1, \cdots,
S_{j_1,c}^1;\cdots;S_{1,c}^{p-1}, \cdots S_{j_{p-1},c}^{p-1}\},$$
where $S_{j,c}^i$ denotes an $i$-dimensional face of $\Omega_c$,
i.e., $S_{j,c}^0$ represents a vertex, $S_{j,c}^1$ an edge, and
$S_{j,c}^{p-1}$ a facet.  It is clear that $\mathcal{F}_c$ is a
finite set.

%To show $\bv_c$ is Lipschitz on $[a, b]$, it is enough to show that
%$\bv_c$ is piecewise linear on [a,b].

Define a mapping $\varphi: [a,b]\rightarrow
\mathbb{Z}\times\mathbb{Z}$, where $\varphi(c)=(i,j)$ such that \rmnum{1})
$\bv_c\in S_{j,c}^i$ and  \rmnum{2}) $i$ is minimal. By definition, this mapping is single
valued.

For any $c_0\in(a,b]$, denote $D_{c_0}=\{(i,j)|\forall
\epsilon>0,\exists c\in[c_0-\epsilon,c_0) \text{ s.t. }
\varphi(c)=(i,j)\}$. The set $D_{c_0}$ is non-empty because the collection
$\{(i,j)\in\mathbb{Z}\times\mathbb{Z}|S^i_{j, c}\in\mathcal{F}_c\}$
is finite. Then the theorem follows from compactness of $[a,b]$ and
Lemma \ref{lemma:D-c0-sigleton}, Remark \ref{remark::linear} and Lemma \ref{lemma:bvc-continuous}.

\begin{lem}\label{lem::piecewiselinear}
$\forall c_0\in (a,b]$, $\exists \epsilon>0$ such that  $\forall
(i,j)\in D_{c_0}$ and  $\forall c\in(c_0-\epsilon, c_0)$, $P^{i}_{j,c}\in
S^{i\circ}_{j,c}\cap S$, where $P^i_{j,c}$ is the projection of $P$ onto
$S\cap \widetilde{S^i_{j,c}}$, and $\widetilde{S^i_{j,c}}$ denotes the $i$-dimensional affine space in which $S^i_{j,c}$ embeds, and $S^{i\circ}_{j,c}$ is the interior of $S^i_{j,c}$, where the topology is the natural subspace topology restricted to $\widetilde{S^i_{j,c}}$.

\end{lem}

\begin{proof}
Fix $c_0\in(a,b]$. For any $(i,j)\in D_{c_0}$ and
$\bar{\epsilon}>0$, by the definition of $D_{c_0}$, there exists $c'\in[c_0-\bar{\epsilon}, c_0)$ such that $\varphi(c')=(i,j)$.  The minimality of $i$ in the definition for $\varphi$ implies that $\bv_{c'}=P^i_{j,c'}\in S^{i\circ}_{j,c'}$, which in the interior of $S^{i}_{j,c'}$.    Therefore, $P^i_{j,c'}\in
S^{i\circ}_{j,c'}\cap S$.  By arbitrariness of $\bar{\epsilon}$, $\exists (c_n)\nearrow c_0$ such that $P^i_{j,c_n}\in S^{i\circ}_{j,c_n}\cap S$ for all $n$.

It can also be shown that $\{c|P^i_{j,c}\in S^{i\circ}_{j,c}\cap S\}$ is
connected:  let $P^i_{j,c'_1}\in S^{i\circ}_{j,c'_1}\cap S$, $P^i_{j,c'_2}\in
S^{i\circ}_{j,c'_2}\cap S$, $c'_1<c'_2$.  For any $c'_3\in(c'_1, c'_2)$,
$P^i_{j,c'_3}$ is on the line segment with endpoints $P^i_{j,c'_1}$
and $P^i_{j,c'_2}$ because $\widetilde{S^i_{j,c}}$ are parallel
affine subspace in $\mathbb{R}^p$. Let $S^i_{j,cone}:=\cup_{c\geq0}S^{i\circ}_{j,c}$, then it is
a cone. Since $P^i_{j,c'_1}\in S^i_{j,cone}$ and $P^i_{j,c'_2}\in
S^i_{j,cone}$, we have $P^i_{j,c'_3}\in S^i_{j,cone}$.  Then,
$P^{i}_{j,c'_3}\in S^i_{j,cone}\cap S\cap
\widetilde{S^i_{j,c'_3}}=S^{i\circ}_{j,c'_3}\cap S$.  Hence, $\exists \epsilon_{ij}>0$ such that for all $c\in[c_0-\epsilon_{ij}, c_0)$, $P^i_{j, c}\in S^{i\circ}_{j, c}$.
Take $\epsilon=\min_{(i,j)\in D_{c_0}}\epsilon_{ij}$, the claim follows.

\end{proof}

\begin{lem}\label{lemma:D-c0-sigleton}
$\forall c_0\in (a,b]$, $D_{c_0}$ is a singleton, and  $\exists \epsilon'>0$ such that $\bv_c$ is
linear in $c$ on $(c_0-\epsilon', c_0)$.

\end{lem}

\begin{proof}
Fix $c_0\in (a,b]$. We
claim that for some $(i,j)\in D_{c_0}$, there exists
 positive $\epsilon'(\leq\epsilon \mbox{ that validates Lemma 1})$ such that for any $
c\in(c_0-\epsilon',c_0)$, $\bv_c=P^i_{j,c}$. Assume that the claim is not correct, then pick any $(i,j)\in D_{c_0}$,   there exists a sequence
$\{c_k\}$ $(c_k\neq c_{k'} \text{ if } k\neq k')$ converging to
$c_0$ from the left s.t. $\bv_{c_k}\neq P^i_{j,c_k}$. Without loss of generality,
take $\{c_{k}\}\subset(c_0-\epsilon, c_0)$. Lemma $1$ implies that $P^i_{j,c_k}\in S^{i\circ}_{j,c_k}\cap S$. If $\bv_{c_k}\in S^i_{j,c_k}$,
we would have $\bv_{c_k}=
P^i_{j,c_k}$.
Hence $\bv_{c_k}\not\in S^i_{j,c_k}$. By finiteness of the index pairs in $\mathcal{F}_c$, there exists
$(i',j')\neq(i,j)$ such that $\varphi(c)=(i',j')$ for
$c\in\{c_{k_l}\}$, where $\{c_{k_l}\}$ is some subsequence of $\{c_k\}$.
This implies $(i',j')\in D_{c_0}$, which together with Lemma 1
implies $\bv_c=P^{i'}_{j',c}$ for $c\in\{c_{k_l}\}$. Therefore
$$\|P^{i'}_{j',c}-P\|_2<\|P^i_{j,c}-P\|_2$$ for $c\in\{c_{k_l}\}$.

On the other hand, because $(i,j)\in D_{c_0}$, there exist
infinitely many $c'\in(c_0-\epsilon, c_0)$ such that
$\|P^{i'}_{j',c'}-P\|_2\geq\|P^i_{j,c'}-P\|_2$.  Therefore,
$$g(c)=\|P-P^i_{j,c}\|_2^2-\|P-P^{i'}_{j',c}\|_2^2$$ changes signs
infinitely many times on $(c_0-\epsilon, c_0)$.
This leads to a contradiction because $P^{i}_{j,c}$ and
$P^{i'}_{j',c}$ are both linear functions of $c$. Hence, the
conclusion holds.

To show that $D_{c_0}$ is a singleton, suppose it has two distinct elements $(i,j)$ and $(i',j')$.  We have shown that $\bv_c=P^i_{j,c}$ and $\bv_c=P^{i'}_{j',c}$ for all $c$ in a left neighborhood of $c_0$ (not including $c_0$). Also we have $P^i_{j,c}\in S^{i\circ}_{j,c}$ and $P^{i'}_{j',c}\in S^{i'\circ}_{j',c}$ by Lemma $1$.  This can be true only when $S^{i\circ}_{j,c}\subset S^{i'\circ}_{j',c}$ (or vice versa), but then $i<i'$, contradicting with minimality in definition of $D_{c_0}$.

\end{proof}

\begin{remark}\label{remark::linear}
Similarly, $\forall c_0\in[a,b)$, $\exists \epsilon'>0$ such that $\bv_c$ is
linear in $c$ on $(c_0, c_0+\epsilon')$.
\end{remark}

\begin{lem}\label{lemma:bvc-continuous}
$\bv_c$ is a continuous function of $c$ on $[a,b]$.
\end{lem}
\begin{proof}
The continuity follows from two parts \rmnum{1}) and \rmnum{2}).

\rmnum{1})  $\forall c_0\in[a,b)$, $\exists \epsilon>0$ such that $\bv_c$ is continuous on $[c_0, c_{0}+\epsilon)$.  Indeed, let
$$h(c)=\min_{\|\bSigma^{-\frac{1}{2}}\bv\|_1\leq c, \bv^T\bar{\bmu_d}=1}\bv^T\bv.$$ We know that the mapping $c\mapsto \bv_c(=P^i_{j,c})$ is linear and hence continuous on $(c_0, c_0+\epsilon)$ for some small $\epsilon>0$. It only remains to show that the mapping is right continuous at $c_0$.  Notice here $h(c)=\|P^i_{j,c}\|_2^2$ for $c\in(c_0,c_{0}+\epsilon)$.  Let $L=\lim_{c\downarrow c_0}P^i_{j,c}$. It is clear that $L\in S^i_{j,c_0}$.  Because $L\in \Omega_{c_0}\cap S$, $h(c_0)\leq \|L\|_2^2$.  This inequality has to take the equal sign because $h(\cdot)$ is monotone decreasing, and $h(c)=\|P^i_{j,c}\|_2^2\rightarrow \|L\|_2^2$ as $c$ approaches $c_0$ from the right. Because $\bv_{c_0}$ is unique, $\bv_{c_0}=L=\lim_{c\downarrow c_0}P^i_{j,c}=\lim_{c\downarrow c_0}\bv_c$.

\rmnum{2})  $\forall c_0\in(a,b]$, $\exists \epsilon>0$ such that $\bv_c$ is continuous on $(c_0-\epsilon, c_0]$. Again, it remains to show that there is no jump at $c_0$.    Let $(i_{c_0}, j_{c_0})=\varphi(c_0)$. Clearly $P^{i_{c_0}}_{j_{c_0}, c_0}\in S^{i_{c_0}\circ}_{j_{c_0},c_0}$. Introduce a notion of parallelism of affine subspaces in $\mathbb{R}^p$.  We denote  $\widetilde{S^{i_{c_0}}_{j_{c_0},c}}\parallel S$, if only by translation, $\widetilde{S^{i_{c_0}}_{j_{c_0},c}}$ becomes a subset of $S$ (or vice versa in other situations); use the notation $\widetilde{S^{i_{c_0}}_{j_{c_0},c}}\nparallel S$ otherwise.

If $\widetilde{S^{i_{c_0}}_{j_{c_0},c}}\nparallel S$, for $c$ in some left neighborhood of $c_0$, $P^{i_{c_0}}_{j_{c_0},c}$ exists and $P^{i_{c_0}}_{j_{c_0},c}\in S^{i_{c_0}\circ}_{j_{c_0}, c}$. Note $P^{i_{c_0}}_{j_{c_0},c}\in \Omega_c\cap S$, and $\|P^{i_{c_0}}_{j_{c_0}, c}\|_2\rightarrow \|P^{i_{c_0}}_{j_{c_0}, c_0}\|_2$ as $c$ approaches $c_0$ from the left.  Since $h(\cdot)$ is monotone decreasing, obviously $h(c)\rightarrow \|P^{i_{c_0}}_{j_{c_0}, c_0}\|_2^2=h(c_0)$. This shows the left continuity of $h$ at $c_0$.   Suppose $D_{c_0}=\{(i,j)\}$, then we know on a left neighborhood of $c_0$ (not including $c_0$), $\bv_c=P^i_{j,c}$.  Let $E=\lim_{c\uparrow c_0}P^i_{j,c}$, then $E\in\Omega_{c_0}\cap S$. Note that $\|P^{i_{c_0}}_{j_{c_0}, c}\|_2\geq \|P^{i}_{j,c}\|_2$ for all $c$ in $c_0$'s left neighborhood, so we have $ \|P^{i_{c_0}}_{j_{c_0},c_0}\|_2\geq \|E\|_2$. On the other hand, $ \|P^{i_{c_0}}_{j_{c_0},c_0}\|_2\leq \|E\|_2$ by the definition of $P^{i_{c_0}}_{j_{c_0},c_0}$.  Also, consider the uniqueness of distance minimizing point in $\Omega_{c_0}\cap S$ to origin $O$, $E=P^{i_{c_0}}_{j_{c_0},c_0}$, and hence $\bv_c$ has left continuity at $c_0$.

%Because $h(c)=\bv_c^T\bv_c$ and $h(c_0)=\bv_{c_0}^T\bv_{c_0}$, $\bv_c\mapsto \bv_{c_0}$.

If $\widetilde{S^{i_{c_0}}_{j_{c_0},c}}\parallel S$, $\exists Q\in \Omega_{c_0-\epsilon/2}\cap S$ such that $Q\neq P^{i_{c_0}}_{j_{c_0}, c_0}$.
%Let $R=OP^{i_{c_0}}_{j_{c_0},c_0}\cap S^{i_{c_0}}_{j_{c_0},c_0-\epsilon/2}$, extend $OQ$ to hit $S^{i_{c_0}}_{j_{c_0},c_0-\epsilon/2}$ at $Z$.
When $c$ goes from $c_0-\epsilon/2$ to $c_0$, there exists a point $Q_c\in \Omega_{c}\cap S$ moving on the line segment from $Q$ to $P^{i_{c_0}}_{j_{c_0}, c_0}$.   Therefore, $h(\cdot)$ is left continuous at $c_0$. Replace $P^{i_{c_0}}_{j_{c_0},c}$ by $Q_c$ in the previous paragraph, the left continuity of $\bv_c$ at $c_0$ follows from the same argument.

\end{proof}
\vspace{2cm}
%\newpage
\bibliographystyle{rss}
\bibliography{bio}

\begin{thebibliography}{39}
\expandafter\ifx\csname natexlab\endcsname\relax\def\natexlab#1{#1}\fi
\expandafter\ifx\csname url\endcsname\relax
  \def\url#1{\texttt{#1}}\fi
\expandafter\ifx\csname urlprefix\endcsname\relax\def\urlprefix{URL }\fi

\bibitem[{Ackermann and Strimmer(2009)}]{Ackermann-2009}
Ackermann, M. and Strimmer, K. (2009) A general modular framework for gene set
  enrichment analysis.
\newblock \emph{BMC Bioinformatics}, \textbf{10}, 47.

\bibitem[{Antoniadis \emph{et~al.}(2003)Antoniadis, Lambert-Lacroix and
  Leblanc}]{Antoniadis-2003}
Antoniadis, A., Lambert-Lacroix, S. and Leblanc, F. (2003) {Effective dimension
  reduction methods for tumor classification using gene expression data}.
\newblock \emph{Bioinformatics}, \textbf{19}, 563--570.

\bibitem[{Bair \emph{et~al.}(2006)Bair, Hastie, Paul and
  Tibshirani}]{BairHast-2006}
Bair, E., Hastie, T., Paul, D. and Tibshirani, R. (2006) Prediction by
  supervised principal components.
\newblock \emph{J. Amer. Statist. Assoc.}, \textbf{101}, 119--137.

\bibitem[{Bickel and Levina(2004)}]{BickelLevina-04}
Bickel, P. and Levina, E. (2004) Some theory for fisher¡¯s linear discriminant
  function, ``naive bayese" and some alternatives when there are many more
  variables than observations.
\newblock \emph{Bernoulli}, \textbf{10}, 989--1010.

\bibitem[{Boulesteix(2004)}]{Boulesteix-2004}
Boulesteix, A.-L. (2004) P{LS} dimension reduction for classification with
  microarray data.
\newblock \emph{Stat. Appl. Genet. Mol. Biol.}, \textbf{3}, Art. 33, 32 pp.
  (electronic).

\bibitem[{Boyd and Vandenberghe(2004)}]{Boyd04}
Boyd, S. and Vandenberghe, L. (2004) \emph{Convex Optimization}.
\newblock Cambridge University Press.

\bibitem[{Breheny and Huang(2011)}]{BrehenyHuang2011}
Breheny, P. and Huang, J. (2011) Coordinate descent algorithms for nonconvex
  penalized regression, with applications to biological feature selection.
\newblock \emph{Annals of Applied Statistics}, \textbf{5}, 232--253.

\bibitem[{Domingos and Pazzani(1997)}]{Domingos-1997}
Domingos, P. and Pazzani, M. (1997) On the optimality of the simple bayesian
  classifier under zero-one loss.
\newblock \emph{Mach. Learn.}, \textbf{29}, 103--130.

\bibitem[{Donoho and Johnstone(1994)}]{Dono:John:idea:1994}
Donoho, D.~L. and Johnstone, I.~M. (1994) Ideal spatial adaptation by wavelet
  shrinkage.
\newblock \emph{Biometrika}, \textbf{81}, 425--455.

\bibitem[{Dudoit \emph{et~al.}(2002)Dudoit, Fridlyand and Speed}]{Dudoit-2002}
Dudoit, S., Fridlyand, J. and Speed, T.~P. (2002) Comparison of discrimination
  methods for the classification of tumors using gene expression data.
\newblock \emph{J. Amer. Statist. Assoc.}, \textbf{97}, 77--87.

\bibitem[{Efron \emph{et~al.}(2004)Efron, Hastie, Johnstone and
  Tibshirani}]{Efron-04}
Efron, B., Hastie, T., Johnstone, I. and Tibshirani, R. (2004) Least angle
  regression.
\newblock \emph{Ann. Statist.}, \textbf{32}, 407--499.

\bibitem[{Fan and Fan(2008)}]{FanFan-08}
Fan, J. and Fan, Y. (2008) High dimensional classification using features
  annealed independence rules.
\newblock \emph{Ann. Statist.}, \textbf{36}, 2605--2637.

\bibitem[{Fan and Li(2001)}]{FanLi2001}
Fan, J. and Li, R. (2001) Variable selection via nonconcave penalized
  likelihood and its oracle properties.
\newblock \emph{J. Amer. Statist. Assoc.}, \textbf{96}, 1348--1600.

\bibitem[{Fan and Lv(2008)}]{FanLv-08}
Fan, J. and Lv, J. (2008) Sure independence screening for ultra-high
  dimensional feature space(with discussion).
\newblock \emph{J. R. Statist. Soc. B}, \textbf{70}, 849--911.

\bibitem[{Fan and Lv(2010)}]{FanLv2010}
Fan, J. and Lv, J. (2010) A selective overview of variable selection in high
  dimensional feature space.
\newblock \emph{Statistica Sinica}, \textbf{20}, 101--148.

\bibitem[{Golub \emph{et~al.}(1999)Golub, Slonim, Tamayo, Huard, Gaasenbeek,
  Mesirov, Coller, Loh, Downing, Caligiuri, Bloomfield and Lander}]{Golub-1999}
Golub, T.~R., Slonim, D.~K., Tamayo, P., Huard, C., Gaasenbeek, M., Mesirov,
  J.~P., Coller, H., Loh, M.~L., Downing, J.~R., Caligiuri, M.~A., Bloomfield,
  C.~D. and Lander, E.~S. (1999) {Molecular Classification of Cancer: Class
  Discovery and Class Prediction by Gene Expression Monitoring}.
\newblock \emph{Science}, \textbf{286}, 531--537.

\bibitem[{Gordon \emph{et~al.}(2002)Gordon, Jensen, Hsiao, Gullans,
  Blumenstock, Ramaswamy, Richards, Sugarbaker and Bueno}]{Gordon-2002}
Gordon, G.~J., Jensen, R.~V., Hsiao, L.-L., Gullans, S.~R., Blumenstock, J.~E.,
  Ramaswamy, S., Richards, W.~G., Sugarbaker, D.~J. and Bueno, R. (2002)
  Translation of microarray data into clinically relevant cancer diagnostic
  tests using gene expression ratios in lung cancer and mesothelioma.
\newblock \emph{Cancer Research}, \textbf{62}, 4963--4967.

\bibitem[{Guo \emph{et~al.}(2005)Guo, Hastie and Tibshirani}]{Guo-2005}
Guo, Y., Hastie, T. and Tibshirani, R. (2005) Regularized discriminant analysis
  and its application in microarrays.
\newblock \emph{Biostatistics}, \textbf{1}, 1--18.

\bibitem[{Hastie \emph{et~al.}(2009)Hastie, Tibshirani and Friedman}]{HTF09}
Hastie, T., Tibshirani, R. and Friedman, J.~H. (2009) \emph{The Elements of
  Statistical Learning: Data Mining, Inference, and Prediction (2nd edition)}.
\newblock Springer-Verlag Inc.

\bibitem[{Huang(2003)}]{HuangPan-2003}
Huang, X., P.~W. (2003) Linear regression and two-class classification with
  gene expression data.
\newblock \emph{Bioinformatics}, \textbf{19}, 2072--2978.

\bibitem[{Krzanowski \emph{et~al.}(1995)Krzanowski, Jonathan, McCarthy and
  Thomas}]{Krzanowski-1995}
Krzanowski, W., Jonathan, P., McCarthy, W. and Thomas, M. (1995) Discriminat
  analysis with singular covariance matrices: methods and applications to
  spectroscopic data.
\newblock \emph{Applied Statistics}, \textbf{44}, 101–115.

\bibitem[{Lewis(1998)}]{Lewis98}
Lewis, D.~D. (1998) Naive (bayes) at forty: The independence assumption in
  information retrieval.
\newblock  4--15. Springer Verlag.

\bibitem[{Li(1991)}]{Li-1991-sliced-inverse-regression}
Li, K.-C. (1991) Sliced inverse regression for dimension reduction.
\newblock \emph{J. Amer. Statist. Assoc.}, \textbf{86}, 316--342.
\newblock With discussion and a rejoinder by the author.

\bibitem[{Nguyen and Rocke(2002)}]{Nguyen-2002}
Nguyen, D.~V. and Rocke, D.~M. (2002) {Tumor classification by partial least
  squares using microarray gene expression data }.
\newblock \emph{Bioinformatics}, \textbf{18}, 39--50.

\bibitem[{Oberthuer \emph{et~al.}(2006)Oberthuer, Berthold, Warnat, Hero,
  Kahlert, Spitz, Ernestus, K\"{o}nig, Haas, Eils, Schwab, Brors, Westermann
  and Fischer}]{Oberthuer06}
Oberthuer, A., Berthold, F., Warnat, P., Hero, B., Kahlert, Y., Spitz, R.,
  Ernestus, K., K\"{o}nig, R., Haas, S., Eils, R., Schwab, M., Brors, B.,
  Westermann, F. and Fischer, M. (2006) Customized oligonucleotide microarray
  gene expression based classification of neuroblastoma patients outperforms
  current clinical risk stratification.
\newblock \emph{Journal of Clinical Oncology}, \textbf{24}, 5070--5078.

\bibitem[{Rosset and Zhu(2007)}]{Rosset&Zhu07}
Rosset, S. and Zhu, J. (2007) Piecewise linear regularized solution paths.
\newblock \emph{Ann. Statist.}, \textbf{35}, 1012--1030.

\bibitem[{Ruszczynski(2006)}]{Ruszczynski06}
Ruszczynski, A. (2006) \emph{Nonlinear Optimization}.
\newblock Princeton University Press.

\bibitem[{Shao \emph{et~al.}(2011)Shao, Wang, Deng and Wang}]{ShaoWang11}
Shao, J., Wang, Y., Deng, X. and Wang, S. (2011) Sparse linear discriminant
  analysis by thresholding for high dimensional data.
\newblock \emph{Ann. Statist.}, \textbf{39}, to appear.

\bibitem[{Tibshirani(1996)}]{Tibs:regr:1996}
Tibshirani, R. (1996) Regression shrinkage and selection via the lasso.
\newblock \emph{J. R. Statist. Soc. B}, \textbf{58}, 267--288.

\bibitem[{Tibshirani \emph{et~al.}(2002)Tibshirani, Hastie, Narasimhan and
  Chu}]{Tibshirani-02}
Tibshirani, R., Hastie, T., Narasimhan, B. and Chu, G. (2002) Diagnosis of
  multiple cancer types by shrunken centroids of gene expression.
\newblock \emph{Proc. Natl. Acad. Sci.}, \textbf{99}, 6567--6572.

\bibitem[{Tseng(2001)}]{Tseng01}
Tseng, P. (2001) Convergence of a block coordinate descent method for
  nondifferentiable minimization.
\newblock \emph{J. Optim. Theory Appl.}, \textbf{109}, 475--494.

\bibitem[{Vapnik(1995)}]{Vapnik-1995}
Vapnik, V.~N. (1995) \emph{The nature of statistical learning theory}.
\newblock New York: Springer-Verlag.

\bibitem[{Wang and Zhu(2007)}]{Wang&Zhu:nsclasso:2007}
Wang, S. and Zhu, J. (2007) Improved centroids estimation for the nearest
  shrunken centroid classifier.
\newblock \emph{Bioinformatics}, \textbf{23}, 972--979.

\bibitem[{Zhang(2010)}]{Zhang09}
Zhang, C.-H. (2010) Nearly unbiased variable selection under minimax concave
  penalty.
\newblock \emph{Ann. Statist.}, \textbf{38}, 894--942.

\bibitem[{Zhao and Li(2010)}]{ZhaoLi10}
Zhao, D.~S. and Li, Y. (2010) Principled sure independence screening for cox
  models with ultra-high-dimensional covariates.
\newblock Manuscript.

\bibitem[{Zou(2006)}]{Zou06}
Zou, H. (2006) The adaptive lasso and its oracle properties.
\newblock \emph{J. Amer. Statist. Assoc.}, \textbf{101}, 1418--1429.

\bibitem[{Zou and Hastie(2005)}]{ZouHastie05}
Zou, H. and Hastie, T. (2005) Regularization and variable selection via the
  elastic net.
\newblock \emph{J. R. Statist. Soc. B}, \textbf{67}, 768--768.

\bibitem[{Zou \emph{et~al.}(2006)Zou, Hastie and
  Tibshirani}]{ZouHastieTibs-2006}
Zou, H., Hastie, T. and Tibshirani, R. (2006) Sparse principal component
  analysis.
\newblock \emph{J. Comput. Graph. Statist.}, \textbf{15}, 265--286.

\bibitem[{Zou and Li(2008)}]{ZouLi08}
Zou, H. and Li, R. (2008) One-step sparse estimates in nonconcave penalized
  likelihood models.
\newblock \emph{Ann. Statist.}, \textbf{36}, 1509--1533.

\end{thebibliography}

\end{document}